\documentclass{article}

\usepackage{microtype}
\usepackage{graphicx}
\usepackage{booktabs} 

\usepackage{hyperref}

\usepackage[accepted]{icml2023}

\usepackage{amsmath}
\usepackage{amssymb}
\usepackage{mathtools}
\usepackage{amsthm}

\usepackage[capitalize,noabbrev]{cleveref}

\theoremstyle{plain}
\newtheorem{theorem}{Theorem}[section]
\newtheorem{proposition}[theorem]{Proposition}
\newtheorem{lemma}[theorem]{Lemma}

\theoremstyle{definition}

\theoremstyle{remark}

\usepackage[textsize=tiny]{todonotes}

\icmltitlerunning{Hyperbolic Sliced-Wasserstein via Geodesic and Horospherical Projections}

\usepackage{bbold} 
\usepackage{graphicx,float,subfig}

\DeclareMathOperator*{\argmin}{argmin}
\DeclareMathOperator*{\argmax}{argmax}
\DeclareMathOperator*{\arccosh}{arccosh}
\DeclareMathOperator*{\arctanh}{arctanh}

\begin{document}

\twocolumn[
\icmltitle{Hyperbolic Sliced-Wasserstein via Geodesic and Horospherical Projections}



\icmlsetsymbol{equal}{*}

\begin{icmlauthorlist}
\icmlauthor{Clément Bonet}{lmba}
\icmlauthor{Laetitia Chapel}{irisa}
\icmlauthor{Lucas Drumetz}{imt}
\icmlauthor{Nicolas Courty}{irisa}
\end{icmlauthorlist}


\icmlaffiliation{lmba}{Université Bretagne Sud, LMBA}
\icmlaffiliation{irisa}{Université Bretagne Sud, IRISA}
\icmlaffiliation{imt}{IMT Atlantique, Lab-STICC}

\icmlcorrespondingauthor{Clément Bonet}{clement.bonet@univ-ubs.fr}

\icmlkeywords{Machine Learning, ICML}

\vskip 0.3in
]



\printAffiliationsAndNotice{}  

\begin{abstract}
    Hyperbolic space embeddings have been shown beneficial for many learning tasks where data have an underlying hierarchical structure. Consequently, many machine learning tools were extended to such spaces, but only few discrepancies to compare probability distributions defined over those spaces exist. Among the possible candidates, optimal transport distances are well defined on such Riemannian manifolds and enjoy strong theoretical properties, but suffer from high computational cost. On Euclidean spaces,  sliced-Wasserstein distances, which leverage a closed-form solution of the Wasserstein distance in one dimension, are more computationally efficient, but are not readily available on hyperbolic spaces. In this work, we propose to derive novel hyperbolic sliced-Wasserstein discrepancies. These constructions use projections on the underlying geodesics either along horospheres or geodesics. We study and compare them on different tasks where hyperbolic representations are relevant, such as sampling or image classification.
\end{abstract}

\section{Introduction}

In recent years, hyperbolic spaces have received a lot of attention in machine learning (ML) as they allow efficiently processing data 
that present a hierarchical structure \citep{nickel2017poincare,nickel2018learning}. This encompasses data such as graphs \citep{gupte2011finding}, words \citep{tifrea2018poincar} or images \citep{khrulkov2020hyperbolic}. Embedding in hyperbolic spaces has been proposed for various applications such as drug embedding \citep{yu2020semi}, image clustering \citep{park2021unsupervised, ghadimi2021hyperbolic}, zero-shot recognition \citep{liu2020hyperbolic}, remote sensing \citep{hamzaoui2021hyperbolic} or reinforcement learning \citep{cetin2022hyperbolic}. Hence, many works proposed to develop tools to be used on such spaces, such as generalization of Gaussian distributions \citep{nagano2019wrapped, galaz2022wrapped}, neural networks \citep{ganea2018hyperbolic, liu2019hyperbolic} or normalizing flows \citep{lou2020neural, bose2020latent}. 

Optimal Transport (OT) \citep{villani2021topics,villani2009optimal} is a popular tool used in ML to compare probability distributions. Among others, it has been used for domain adaptation \citep{courty2016optimal}, learning generative models \citep{arjovsky2017wasserstein} or document classification \citep{kusner2015word}. However, the main tool of OT is the Wasserstein distance which exhibits an expensive, super-cubical computational cost \emph{w.r.t.} the number of samples of each distribution. Hence, many workarounds have been proposed to alleviate the computational burden such as  entropic regularization \citep{cuturi2013sinkhorn}, minibatch OT \citep{pmlr-v108-fatras20a} or the sliced-Wasserstein (SW) distance \citep{rabin2011wasserstein}. In particular, SW is a popular variant of the Wasserstein distance that computes the expected distance between one dimensional projections on some lines of the two distributions. Its computational advantages and theoretical properties make it an efficient and popular alternative to the Wasserstein distance. For example, it has been used for texture synthesis \citep{heitz2021sliced} 
or for generative modeling with SW autoencoders \citep{kolouri2018sliced}, SW GANs \citep{deshpande2018generative}, SW flows \citep{liutkus2019sliced} or SW gradient flows \citep{bonet2021sliced}.

\looseness=-1 The theoretical study of the Wasserstein distance on Riemannian manifolds is well developed \citep{mccann2001polar, villani2009optimal}. When it comes to hyperbolic spaces, some optimal transport attempts aimed at aligning distributions of data which have been embedded in a hyperbolic space \citep{alvarez2020unsupervised,hoyos2020aligning}. Regarding SW, it is originally defined using Euclidean distances and projections, which are not well suited to other manifolds. Recently, \citet{rustamov2020intrinsic} proposed to defined a SW distance on compact manifolds using the eigendecomposition of the Laplace-Beltrami operator while \citet{bonet2022spherical} proposed a SW distance to tackle this problem for measures supported on the sphere by using only objects intrinsically defined on this specific manifold. Contrary to the elliptical geometry of the sphere, the negative curvature of hyperbolic spaces calls for drastically different strategies to define geodesics and the associated projection operators. This work proposes to close this gap by proposing new SW constructions on these spaces. 

\textbf{Contributions.} \looseness=-1 We extend sliced-Wasserstein to data living in hyperbolic spaces. Analogously to Euclidean SW, we project the distributions on geodesics passing through the origin. Interestingly enough, different projections can be considered, leading to several new SW constructions that exhibit different theoretical properties and empirical benefits. We make connections with Radon transforms already defined in the literature and we show that hyperbolic SW are (pseudo-) distances. We provide the algorithmic procedure and discuss its complexity. We illustrate the benefits of these new hyperbolic SW distances on several tasks such as sampling or image classification.

\section{Background}

\looseness=-1 In this Section, we first provide some background on Optimal Transport with the Wasserstein and the sliced-Wasserstein distance. We then review two common hyperbolic models, namely the Lorentz and Poincaré ball models, on which we will define new OT discrepancies in the next section.

\subsection{Optimal Transport} \label{sec:ot}

Optimal transport is a popular field which allows  comparing  distributions of probabilities by determining a transport plan minimizing some ground cost. The main tool of OT is the Wasserstein distance which we introduce now.


\paragraph{Wasserstein Distance on Riemannian Manifolds.}

Let $M$ be a Riemannian manifold endowed with a Riemannian distance $d$. For $p\ge 1$, the $p$-Wasserstein distance between two probability measures $\mu,\nu\in\mathcal{P}_p(M)=\{\mu\in\mathcal{P}(M),\ \int_{M} d(x,x_0)^p\ \mathrm{d}\mu(x)<\infty \text{ for any }x_0\in M\}$ is defined as
\begin{equation}
    W_p^p(\mu,\nu) = \inf_{\gamma\in\Pi(\mu,\nu)}\ \int_{M\times M} d(x,y)^p\ \mathrm{d}\gamma(x,y),
\end{equation}
where $\Pi(\mu,\nu)=\{\gamma\in\mathcal{P}(M\times M),\ \pi^1_\#\gamma=\mu,\ \pi^2_\#\gamma=\nu\}$ is the set of couplings, $\pi^1(x,y)=x$, $\pi^2(x,y)=y$ and $\#$ is the pushforward operator defined as, for all borelian $A\subset M$, $T_\#\mu(A)=\mu(T^{-1}(A))$. For more details about OT, we refer to \citep{villani2009optimal}.

\looseness=-1 The main bottleneck of the Wasserstein distance is its computational complexity. Indeed, for two discrete probability measures with $n$ samples, it can be solved using linear programs \citep{peyre2019computational} with a complexity of $O(n^3\log n)$, which prevents its use when large amount of data are at stake. Hence, a whole literature consists at deriving alternative OT metrics with a smaller computational cost.


\paragraph{Sliced-Wasserstein Distance on Euclidean Space.}

\looseness=-1 On Euclidean spaces, a popular proxy of the Wasserstein distance is the so-called sliced-Wasserstein distance. On the real line, for $p\ge 1$, the $p$-Wasserstein distance between $\mu,\nu\in\mathcal{P}_p(\mathbb{R})$ admits the following closed-form \citep[Remark 2.30]{peyre2019computational} :
\begin{equation}
    W_p^p(\mu,\nu) = \int_0^1 | F_{\mu}^{-1}(u)-F_{\nu}^{-1}(u)|^p\ \mathrm{d}u
\end{equation}
where $F_\mu^{-1}$ and $F_\nu^{-1}$ denote the quantile functions of $\mu$ and $\nu$. This can be approximated in practice very efficiently as it only requires to sort the samples, which has a complexity of $O(n\log n)$. Therefore, \citet{rabin2011wasserstein} defined the sliced-Wasserstein distance by projecting linearly the probabilities on all the possible directions. For a direction $\theta\in S^{d-1}$, denote, for all $x\in\mathbb{R}^d$, $P^\theta(x)=\langle x,\theta\rangle$ the projection in direction $\theta$, and $\lambda$ the uniform measure on $S^{d-1}$. Then, the SW distance between $\mu,\nu\in\mathcal{P}_p(\mathbb{R}^d)$ is defined as
\begin{equation}
    SW_p^p(\mu,\nu) = \int_{S^{d-1}} W_p^p(P^\theta_\#\mu,P^\theta_\#\nu)\ \mathrm{d}\lambda(\theta).
\end{equation}
Using a Monte-Carlo approximation, this can be approximated in $O(Ln(d+\log n))$ where $L$ is the number of projections and $n$ the number of samples.

Moreover, the slicing process has many appealing properties, such as having a sample complexity independent of the dimension \citep{nadjahi2020statistical}, being topologically equivalent to Wasserstein \citep{bonnotte2013unidimensional} and being an actual distance. For the latter point, it can be shown to be a pseudo-distance using that $W_p$ is a distance. The indiscernible property relies on the link between the projection used in SW and the Radon transform \citep{bonneel2015sliced, kolouri2019generalized} which is injective on the space of measures \citep[Theorem A]{boman2009support}. More precisely, let $f\in L^1(\mathbb{R}^d)$, then its Radon transform $R:L^1(\mathbb{R}^d)\to L^1(\mathbb{R}\times S^{d-1})$ is defined for $t\in\mathbb{R}$, $\theta\in S^{d-1}$ as,
\begin{equation}
    Rf(t,\theta) = \int_{\mathbb{R}^d} f(x) \mathbb{1}_{\{\langle x,\theta\rangle = t\}}\ \mathrm{d}x.
\end{equation}
This transform admits a dual operator $R^*:C_0(\mathbb{R}\times S^{d-1})\to C_0(\mathbb{R}^d)$, with $C_0(\mathbb{R}\times S^{d-1})$ the set of continuous functions that vanish at infinity, such that for all $g\in C_0(\mathbb{R}\times S^{d-1})$, $\langle Rf,g\rangle_{\mathbb{R}\times S^{d-1}} = \langle f, R^*g\rangle_{\mathbb{R}^d}$ \citep{bonneel2015sliced}. This allows defining the Radon transform of a measure $\mu\in\mathcal{M}(\mathbb{R}^d)$ as the measure $R\mu\in\mathcal{M}(\mathbb{R}\times S^{d-1})$ satisfying for all $g\in C_0(\mathbb{R}\times S^{d-1})$, $\langle R\mu,g\rangle_{\mathbb{R}\times S^{d-1}}=\langle \mu,R^*g\rangle_{\mathbb{R}^d}$ \citep{boman2009support}. Then, it was shown in \citep{bonneel2015sliced} that, by denoting by $(R\mu)^\theta$ the disintegration \emph{w.r.t.} to the uniform distribution on $S^{d-1}$, 
\begin{equation}
    SW_p^p(\mu,\nu) = \int_{S^{d-1}} W_p^p\big((R\mu)^\theta, (R\nu)^\theta\big)\ \mathrm{d}\lambda(\theta).
\end{equation}
Therefore, $SW_p^p(\mu,\nu) = 0$ implies that, for $\lambda$-ae $\theta$, $(R\mu)^\theta = (R\nu)^\theta$, which implies that $\mu=\nu$ by injectivity of the Radon transform on measures.

\looseness=-1 Many variants of this distance were recently proposed. Most lines of work considered different subspaces for projecting the data: hypersurfaces \citep{kolouri2019generalized}, Hilbert curves \citep{li2022hilbert} or subspace of higher dimensions \citep{lin2020projection,lin2021projection}. When it comes to data living on Riemannian manifolds, \citet{rustamov2020intrinsic} defined a variant on compact manifolds and \citet{bonet2022spherical} extended SW for spherical data. 

\subsection{Hyperbolic Spaces}

\looseness=-1 Hyperbolic spaces are Riemannian manifolds of negative constant curvature \citep{lee2006riemannian}. They have received recently a surge of interest in machine learning as they allow embedding efficiently data with a hierarchical structure \citep{nickel2017poincare,nickel2018learning}. A thorough review of the recent use of hyperbolic spaces in machine learning can be found in \citep{peng2021hyperbolic}.

There are five usual parameterizations of a hyperbolic manifold \citep{peng2021hyperbolic}. They are equivalent (isometric) and one can easily switch from one formulation to the other. Hence, in practice, we use the one which is the most convenient, either given the formulae to derive or the numerical properties. In machine learning, the two most used models are the Poincaré ball and the Lorentz model (also known as the hyperboloid model). Each of these models has its own advantages compared to the other. For example, the Lorentz model has a distance which behaves better \emph{w.r.t.} numerical issues compared to the distance of the Poincaré ball. 
However, the Lorentz model is unbounded, contrary to the Poincaré ball. We introduce in the following these two models as we will use both of them 
in our work.

\textbf{Lorentz model.} First, we introduce the Lorentz model $\mathbb{L}^d\subset \mathbb{R}^{d+1}$ of a $d$-dimensional hyperbolic space. It can be defined as
\begin{equation}
    \mathbb{L}^d = \{(x_0,\dots,x_{d+1})\in \mathbb{R}^{d},\ \langle x,x\rangle_\mathbb{L}=-1, x_0>0\}
\end{equation}
where
\begin{equation}
    \forall x,y\in\mathbb{R}^{d+1},\ \langle x,y\rangle_\mathbb{L} = -x_0y_0 + \sum_{i=1}^{d} x_iy_i
\end{equation}
is the Minkowski pseudo inner-product \citep[Chapter 7]{boumal2022intromanifolds}. The Lorentz model can be seen as the upper sheet of a two-sheet hyperboloid. In the following, we will denote $x^0 = (1,0,\dots,0)\in\mathbb{L}^d$ the origin of the hyperboloid. The geodesic distance in this manifold, which denotes the length of the shortest path between two points, can be defined as
\begin{equation}
    \forall x,y\in\mathbb{L}^d,\ d_{\mathbb{L}}(x,y) = \arccosh(-\langle x,y\rangle_\mathbb{L}).
\end{equation}

At any point $x\in\mathbb{L}^d$, we can associate a subspace of $\mathbb{R}^{d+1}$ orthogonal in the sense of the Minkowski inner product. These spaces are called tangent spaces and are described formally as $T_x\mathbb{L}^d=\{v\in\mathbb{R}^{d+1},\ \langle v,x\rangle_\mathbb{L}=0\}$. Note that on tangent spaces, the Minkowski inner-product is a real inner product. In particular, on $T_{x^0}\mathbb{L}^d$, it is the usual Euclidean inner product, \emph{i.e.} for $u,v\in T_{x^0}\mathbb{L}^d$, $\langle u,v\rangle_\mathbb{L}=\langle u,v\rangle$. Moreover, for all $v\in T_{x^0}\mathbb{L}^d$, $v_0=0$.

We can draw a connection with the sphere. Indeed, by endowing $\mathbb{R}^{d+1}$ with $\langle\cdot,\cdot\rangle_\mathbb{L}$, we obtain $\mathbb{R}^{1,d}$ the so-called Minkowski space. Then, $\mathbb{L}^d$ is the analog in the Minkowski space of the sphere $S^{d}$ in the regular Euclidean space \citep{bridson2013metric}.

\textbf{Poincaré ball.} The second model of hyperbolic space we will be interested in is the Poincaré ball $\mathbb{B}^d\subset\mathbb{R}^d$. This space can be obtained as the stereographic projection of each point $x\in\mathbb{L}^d$ onto the hyperplane $\{x\in\mathbb{R}^{d+1},\ x_0=0\}$. More precisely, the Poincaré ball is defined as
\begin{equation}
    \mathbb{B}^d = \{x\in\mathbb{R}^d,\ \|x\|_2 < 1\},
\end{equation}
with geodesic distance, for all $x,y\in\mathbb{B}^d$,
\begin{equation}
    d_\mathbb{B}(x,y) = \arccosh\left(1+2\frac{\|x-y\|_2^2}{(1-\|x\|_2^2)(1-\|y\|_2^2)}\right).
\end{equation}
We see on this formulation that the distance can be subject to numerical instabilities when one of the points is too close to the boundary of the ball.
\par
We can switch from Lorentz to Poincaré using the following isometric projection \citep{nickel2018learning}:
\begin{equation}
    \forall x\in\mathbb{L}^d,\ P_{\mathbb{L}\to\mathbb{B}}(x) = \frac{1}{1+x_0} (x_1,\dots, x_d)
\end{equation}
and from Poincaré to Lorentz by
\begin{equation}
    \forall x \in \mathbb{B}^d,\ P_{\mathbb{B}\to\mathbb{L}}(x) = \frac{1}{1-\|x\|_2^2}(1+\|x\|_2^2, 2x_1,\dots, 2x_d).
\end{equation}

\section{Hyperbolic Sliced-Wasserstein Distances}

\begin{figure*}[t]
    \centering
    \hspace*{\fill}
    \subfloat[Euclidean.]{\label{fig:proj_euc}\includegraphics[width={0.15\linewidth}]{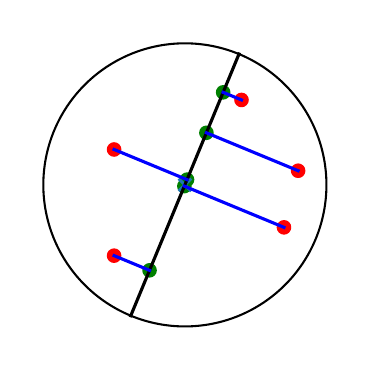}} \hfill
    \subfloat[Geodesics.]{\label{fig:proj_geods}\includegraphics[width={0.15\linewidth}]{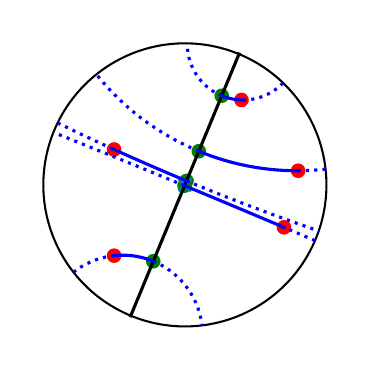}} \hfill
    \subfloat[Horospheres.]{\label{fig:proj_horo}\includegraphics[width={0.15\linewidth}]{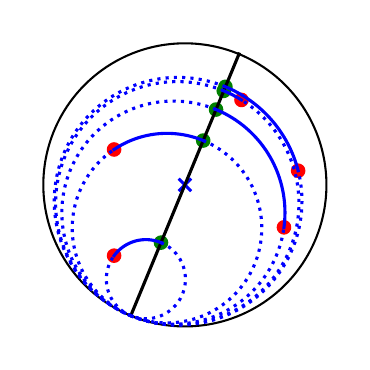}} \hfill
    \subfloat[Euclidean.]{\label{fig:proj_euc_l}\includegraphics[width=0.15\linewidth]{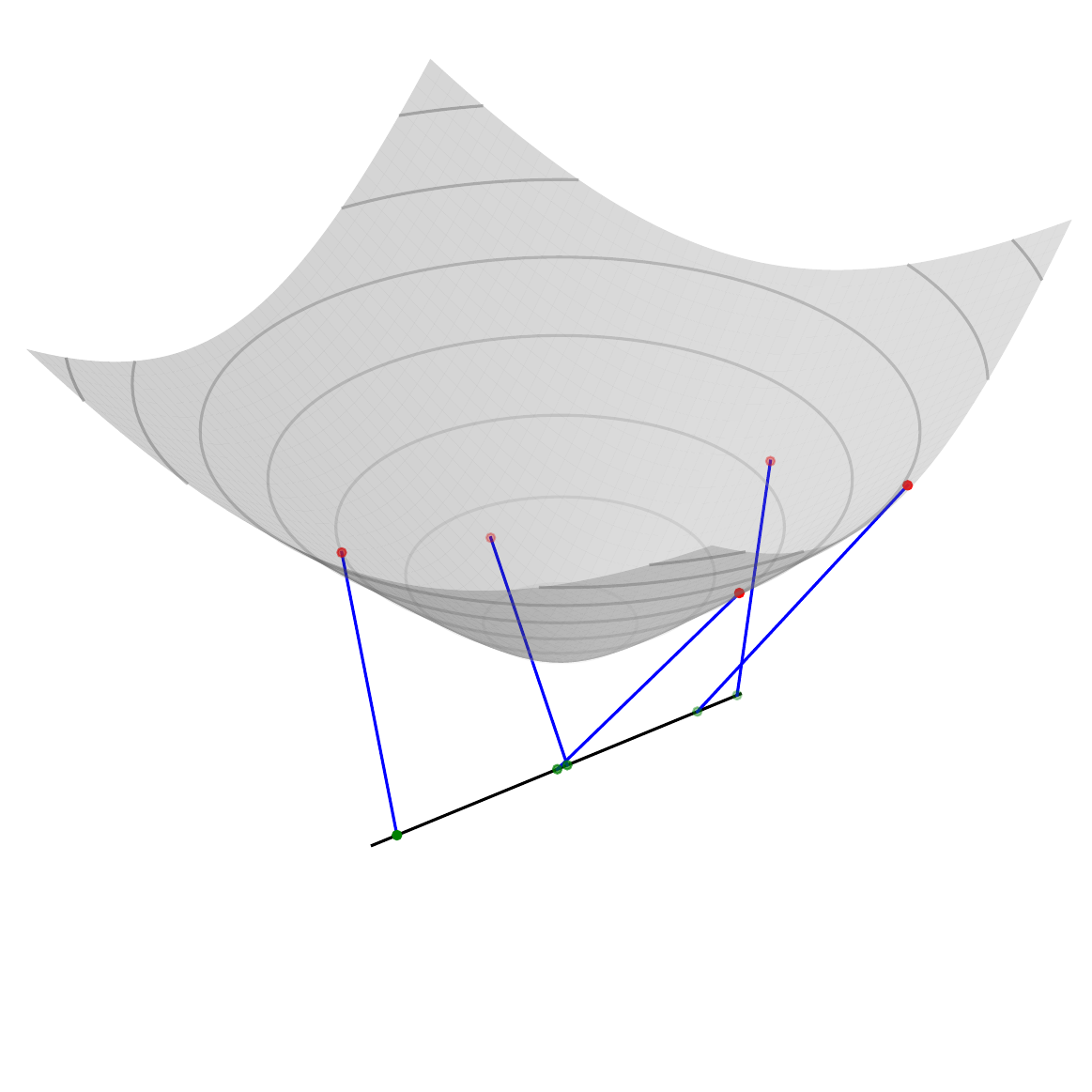}} 
    \subfloat[Geodesics.]{\label{fig:proj_lorentz}\includegraphics[width=0.15\linewidth]{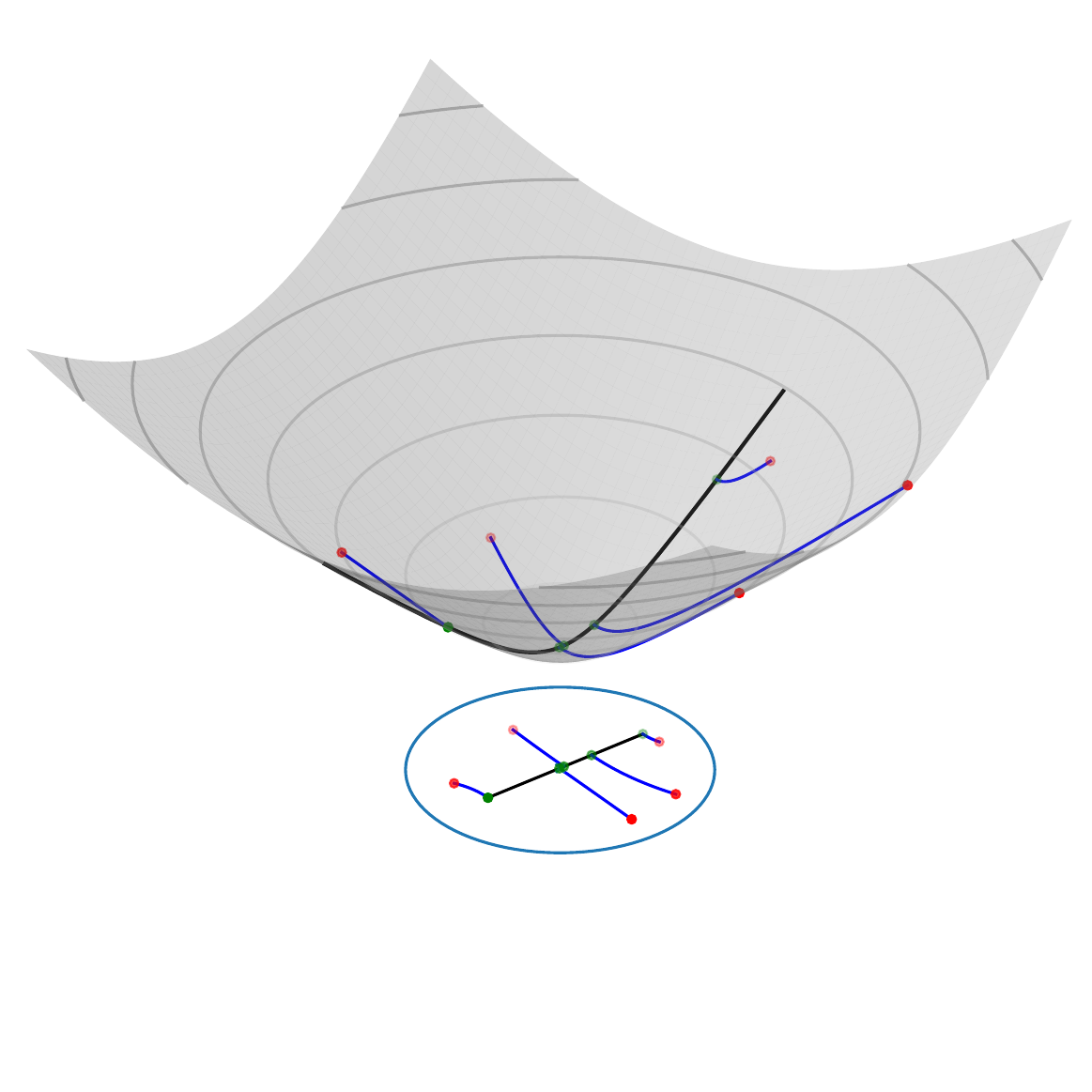}} \hfill
    \subfloat[Horospheres.]{\label{fig:proj_lorentz_horo}\includegraphics[width=0.15\linewidth]{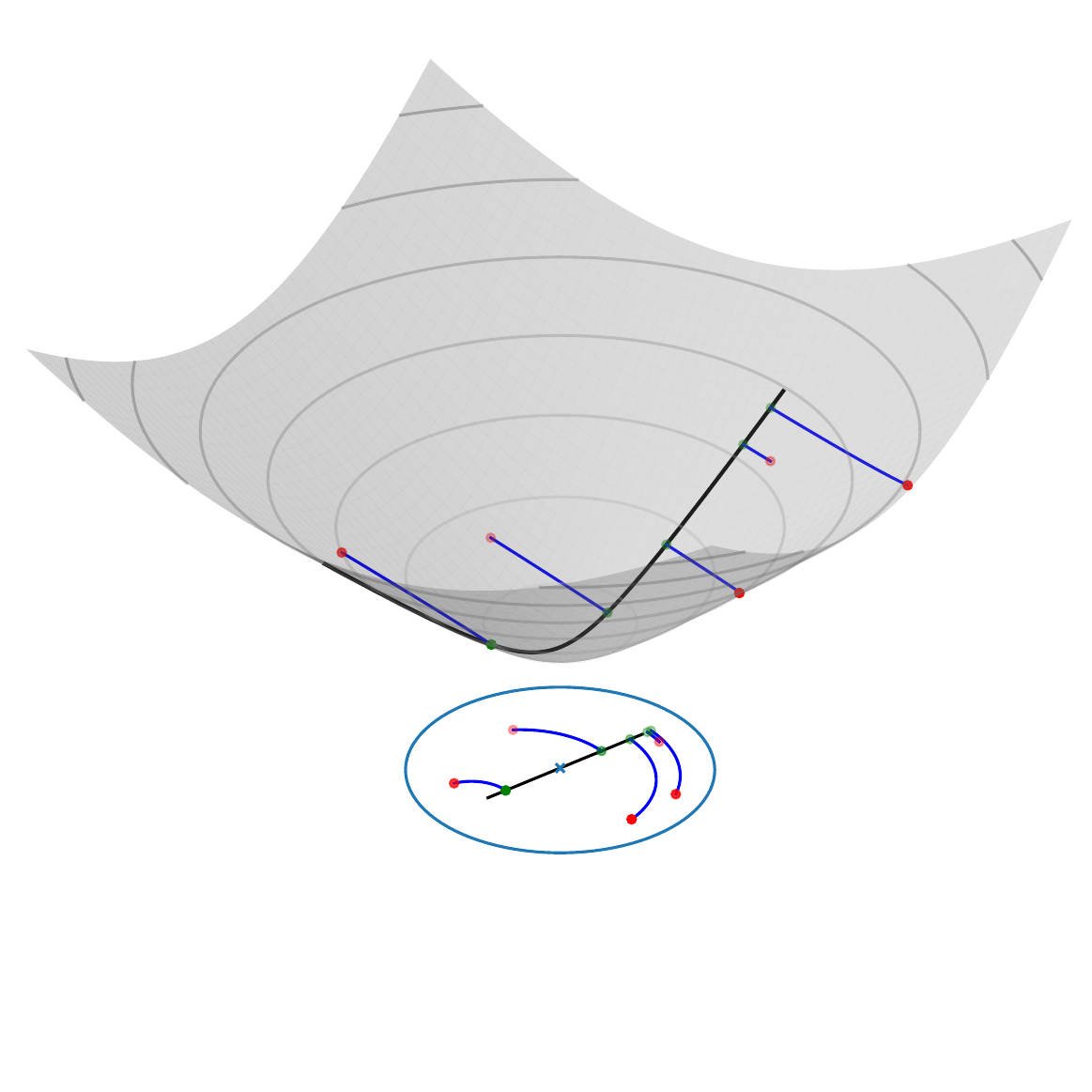}} \hfill
    \hspace*{\fill}
    \caption{Projection of (red) points on a geodesic (black line) in the Poincaré ball and in the Lorentz model along  Euclidean lines, geodesics or horospheres (in blue). Projected points on the geodesic are in green.}
    \label{fig:projections}
    \vspace{-10pt}
\end{figure*}

In this work, we aim at introducing sliced-Wasserstein type of distances on hyperbolic spaces. Interestingly enough, several constructions can be performed, depending on the projections that are involved. The first solution we consider is the extension of Euclidean SW between distributions whose support lies on hyperbolic spaces. We also provide variants that involve a geodesic cost. To do so, we first define the subspace on which the Wasserstein distance can be efficiently computed and then provide two different projection operators: geodesic and horospherical. We finally define the related hyperbolic sliced-Wasserstein distances and discuss some of their properties. 
All the proofs are reported in Appendix \ref{appendix:proofs}.

\subsection{Euclidean Sliced-Wasserstein on Hyperbolic Spaces}

\looseness=-1 The support of distributions lying on hyperbolic space are included in
the ambient spaces $\mathbb{R}^{d}$ (Poincaré ball) or $\mathbb{R}^{d+1}$ (Lorentz model). As such, Euclidean SW can be used for such kind of data. On the Poincaré ball, the projections lie onto the manifold as geodesics passing through the origin are straight lines (see Section \ref{sec:proj}), but the initial geometry of the data might not be fully taken care of as the orthogonal projection does not respect the Poincaré geodesics. On the Lorentz model though, the projections lie out of the manifold. We will denote SWp and SWl the Poincaré ball and Lorentz model version. These formulations allow inheriting from the properties of SW, such as being a distance.

\subsection{Projection Set and Wasserstein Distance} \label{sec:proj}

To generalize the sliced-Wasserstein distance on other spaces, we first define on which subspace to project. Euclidean spaces can be seen as Riemannian manifolds of null constant curvature whose geodesics are straight lines. Therefore, analogously to the Euclidean space 
, we project on geodesics passing through the origin. 
We now describe geodesics in the Lorentz model and in the Poincaré ball.

\textbf{Geodesics.} In the Lorentz model, geodesics passing through the origin $x^0$ can be obtained by taking the intersection between $\mathbb{L}^d$ and a 2-dimensional plane containing $x^0$ \citep[Proposition 5.14]{lee2006riemannian}. Any such plane can be obtained as $\mathrm{span}(x^0,v)$ where $v\in T_{x^0}\mathbb{L}^d\cap S^{d} = \{v\in S^d,\ v_0=0\}$. The corresponding geodesic can be described by a geodesic line \citep[Corollary 2.8]{bridson2013metric}, \emph{i.e.} a map $\gamma:\mathbb{R}\to\mathbb{L}^d$ satisfying for all $t,s\in\mathbb{R}$, $d_\mathbb{L}(\gamma(s),\gamma(t)) = |t-s|$, of the form 
\begin{equation}
    \forall t\in\mathbb{R},\ \gamma(t) = \exp_{x^0}(tv) = \cosh(t) x^0 + \sinh(t) v.
\end{equation}

On the Poincaré ball, geodesics are circular arcs perpendicular to the boundary $S^{d-1}$ \citep[Proposition 5.14]{lee2006riemannian}. In particular, geodesics passing through the origin are straight lines. Hence, they can be characterized by a point $\Tilde{v}$ on the border $S^{d-1}$. Such points will be called ideal points.

\textbf{Wasserstein distance on geodesics.} In order to have an efficient way to compute the discrepancy, we need a practical way to compute the Wasserstein distance on geodesics. As the distance between any point on a geodesic line $\gamma$ and the origin can take arbitrary values on $\mathbb{R}_+$, we project points from the geodesic to the real line $\mathbb{R}$. Indeed, on $\mathbb{R}$, there exists a well known closed-form (see Section \ref{sec:ot}) that can be efficiently computed in practice. In the Lorentz model, let $v\in T_{x^0}\mathbb{L}^d\cap S^d$ be a direction such that $\gamma(\mathbb{R})=\mathbb{L}^d\cap \mathrm{span}(x^0,v)$. Then, we propose to project a point $x\in\gamma(\mathbb{R})$ using
\begin{equation} \label{eq:proj_geod_r}
  t_\mathbb{L}^v(x) = \mathrm{sign}(\langle x,v\rangle) d_\mathbb{L}(x,x^0).  
\end{equation}
The scalar product with $v$ gives an orientation to the geodesic, and the distance to the origin the coordinate of $x$. We can do the same on the Poincaré ball with $t^{\Tilde{v}}_\mathbb{B}(x) = \mathrm{sign}(\langle x,\Tilde{v}\rangle) d_\mathbb{B}(x,0)$, where $\Tilde{v}$ is one of the ideal point to which the geodesic is perpendicular. In the remainder, we will remove the subscripts $\mathbb{L}$ and $\mathbb{B}$ when it is clear from the context. Finally, we need to check that this projection keeps the geodesic Wasserstein distance unchanged. We formulate the following proposition in the Lorentz model.

\begin{proposition}[Wasserstein distance on geodesics.] \label{prop:wasserstein_geodesics}
    Let $v\in T_{x^0}\mathbb{L}^d\cap S^{d}$ and $\mathcal{G}=\mathrm{span}(x^0, v)\cap \mathbb{L}^d$ a geodesic passing through $x^0$. Then, for $p\ge 1$ and $\mu,\nu\in\mathcal{P}_p(\mathcal{G})$, 
    \begin{equation}
        \begin{aligned}
            W_p^p(\mu,\nu) &= W_p^p(t^v_\#\mu, t^v_\#\nu) \\
            &= \int_0^1 |F_{t^v_\#\mu}^{-1}(u) - F_{t^v_\#\nu}^{-1}(u)|^p \ \mathrm{d}u.
        \end{aligned}
    \end{equation}
\end{proposition}


The last ingredient of hyperbolic SW is the way the points lying in the manifold are projected onto the geodesic. We introduce here two different projections that are illustrated on Figure \ref{fig:projections}.

\subsection{Hyperbolic Sliced-Wasserstein}

\textbf{With geodesic projections.} We discuss here the results in the Lorentz model, but we can also obtain all the results in the Poincaré ball. Let $v\in T_{x^0}\cap S^d$ and $\mathcal{G}^v = \{\exp_{x^0}(tv),\ t\in\mathbb{R}\}$ a geodesic passing through $x^0$. As a first generalization of the sliced-Wasserstein distance on hyperbolic spaces, we propose to use the geodesic projection $\Tilde{P}^v$, which projects points on $\mathcal{G}^v$ following the shortest path (geodesics), and which is defined as
\begin{equation}
    \forall x\in \mathbb{L}^d,\ \Tilde{P}^v(x) = \argmin_{y\in \mathcal{G}^v}\ d(x,y).
\end{equation}
We report in Appendix \ref{appendix:geod_proj} the closed-form formulas on both the Lorentz model and the Poincaré ball. Here, we are mostly interested into the coordinate on $\mathbb{R}$, which can be obtained either by computing $t^v\circ \Tilde{P}^v$, or as
\begin{equation}
    \forall x\in \mathbb{L}^d,\ P^v(x) = \argmin_{t\in\mathbb{R}}\ d_\mathbb{L}\big(\exp_{x^0}(tv), x\big).
\end{equation}
Regarding the implementation, we derive a closed-form in the following proposition.
\begin{proposition}[Coordinate of the geodesic projection] \label{prop:hsw_coord_geod_proj} \leavevmode
    \begin{enumerate}
        \item Let $\mathcal{G}^v = \mathrm{span}(x^0, v)\cap \mathbb{L}^d$ where $v\in T_{x^0}\mathbb{L}^d\cap S^d$. Then, the coordinate $P^v$ of the geodesic projection on $\mathcal{G}^v$ of $x\in \mathbb{L}^d$ is
        \begin{equation}
            P^v(x) = \arctanh\left(-\frac{\langle x, v\rangle_\mathbb{L}}{\langle x,x^0\rangle_\mathbb{L}}\right).
        \end{equation}
        \item Let $\Tilde{v}\in S^{d-1}$ be an ideal point. Then, the coordinate $P^{\Tilde{v}}$ of the geodesic projection on the geodesic characterized by $\Tilde{v}$ of $x\in \mathbb{B}^d$ is
        \begin{equation}
            P^{\Tilde{v}}(x) = 2 \arctanh\big(s(x)\big),
        \end{equation}
        where 
        \begin{equation}
            s(x) = \left\{\begin{array}{ll} \frac{1+\|x\|_2^2 - \sqrt{(1+\|x\|_2^2)^2 - 4 \langle x, \Tilde{v}\rangle^2}}{2 \langle x, \Tilde{v}\rangle} & \mbox{ if } \langle x,\Tilde{v}\rangle \neq 0 \\
            0 & \mbox{ if } \langle x,\Tilde{v}\rangle = 0.
            \end{array}\right.
        \end{equation}
    \end{enumerate}
\end{proposition}

Now, we have all the tools to define the geodesic hyperbolic sliced-Wasserstein discrepancy (GHSW) between $\mu,\nu\in\mathcal{P}_p(\mathbb{L}^d)$ as, for $p \ge 1$,
\begin{equation}
    GHSW_p^p(\mu,\nu) = \int_{T_{x^0}\mathbb{L}^d \cap S^{d}} W_p^p(P^v_\#\mu,P^v_\#\nu)\ \mathrm{d}\lambda(v).
\end{equation}
Note that $T_{x^0}\mathbb{L}^d\cap S^d \cong S^{d-1}$ and that $v$ can be drawn by first sampling $\Tilde{v}\sim\mathrm{Unif}(S^{d-1})$ and then adding a $0$ in the first coordinate, \emph{i.e.} $v=(0,\Tilde{v})$ with $\Tilde{v}\in S^{d-1}$. Note also that $GHSW_p(\mu,\nu)<\infty$ for $\mu,\nu\in\mathcal{P}_p(\mathbb{L}^d)$. We also have the Poincaré formulation using $P^{\Tilde{v}}$, and defined between $\mu,\nu\in\mathcal{P}(\mathbb{B}^d)$ as
\begin{equation}
    GHSW_p^p(\mu,\nu) = \int_{S^{d-1}} W_p^p(P^{\Tilde{v}}_\#\mu, P^{\Tilde{v}}_\#\nu)\ \mathrm{d}\lambda(\Tilde{v}).
\end{equation}

\textbf{With horospherical projections.} As we saw in Section \ref{sec:ot}, the projection on geodesics in the Euclidean space is obtained by taking the inner product. A first viewpoint is to see it as the geodesic projection of $x\in\mathbb{R}^d$ on the geodesic $\mathrm{span}(\theta)$:
\begin{equation}
    \langle x,\theta\rangle \theta = \argmin_{y\in\mathrm{span}(\theta)}\ \|x-y\|_2.
\end{equation}
In this case, using a similar projection as \eqref{eq:proj_geod_r}, the coordinates on the line are obtained as the inner product: 
\begin{equation}
    t^\theta(x)=\mathrm{sign}(\langle x,\theta\rangle) \|\langle x,\theta\rangle\theta-0\|_2 = \langle x,\theta\rangle.
\end{equation}
However, the inner product $\langle x, \theta\rangle$ can actually also be seen directly as a coordinate on the line $\mathrm{span}(\theta)$. This can be translated by the Busemann function on unit-speed geodesics, which can be generalized on certain Riemannian manifolds. More precisely, the Busemann function associated to the geodesic ray $\gamma$, \emph{i.e.} a geodesic from $\mathbb{R}_+$ to the manifold satisfying $d(\gamma(t),\gamma(s)) = |t-s|$, is defined as \citep[Definition 8.17]{bridson2013metric} 
\begin{equation}
    B^\gamma(x) = \lim_{t\to\infty}\ \big( d(x,\gamma(t))- t\big),
\end{equation}
where $x$ belongs to the corresponding manifold and $d$ is the geodesic distance. It can be checked that on Euclidean spaces, $B^{\mathrm{span}(\theta)}(x) = -\langle x,\theta\rangle$. While the Busemann function is not well defined on positively curved spaces such as the sphere (as geodesics are periodic), closed-form are available on hyperbolic spaces and provide different projections. We report them in the next proposition. 
As we only work with geodesics passing through the origin, we put as indices the directions which fully characterize them (either $v\in T_{x^0}\mathbb{L}^d$ in $\mathbb{L}^d$, or $\Tilde{v}\in S^{d-1}$ in $\mathbb{B}^d$). 

\begin{proposition}[Busemann function on hyperbolic space] \label{prop:busemann_closed_forms} \leavevmode
    \vspace{-2.2em} 
    \begin{enumerate}
        \item On $\mathbb{L}^d$, for any direction $v\in T_{x^0}\mathbb{L}^d\cap S^d$, 
        \begin{equation}
            \forall x\in\mathbb{L}^d,\ B^v(x) = \log(-\langle x,x^0+v\rangle_\mathbb{L}).
        \end{equation}
        \item On $\mathbb{B}^d$, for any ideal point $\Tilde{v}\in S^{d-1}$,
        \begin{equation}
            \forall x\in \mathbb{B}^d,\ B^{\Tilde{v}}(x) = \log\left(\frac{\|\Tilde{v}-x\|_2^2}{1-\|x\|_2^2}\right).
        \end{equation}
    \end{enumerate}
\end{proposition}

To conserve Busemann coordinates, it has been proposed by \citet{chami2021horopca} to project points on a subset following the level sets of the Busemann function. Those  level sets 
are known as horospheres, which can be seen as spheres of infinite radius \citep{izumiya2009horospherical}. In the Poincaré ball, a horosphere is a Euclidean sphere tangent to an ideal point. \citet{chami2021horopca} argued that this projection is beneficial against the geodesic projection as it tends to better preserve the distances. This motivates us to project on geodesics following the level sets of the Busemann function in order to conserve the Busemann coordinates, \emph{i.e.} we want to have $B^{\Tilde{v}}(x) = B^{\Tilde{v}}(P^{\Tilde{v}}(x))$ (resp. $B^v(x)=B^v(P^v(x))$) on the Poincaré ball (resp. Lorentz model) where $\Tilde{v}\in S^{d-1}$ (resp. $v\in T_{x^0}\mathbb{L}^d\cap S^d$) is characterizing the geodesic. We report the closed-forms in Appendix \ref{appendix:busemann_closed_forms}. In practice, noting that $B^\gamma(x)=B^\gamma(\gamma(t))=-t$, we obtain that the coordinate is $t=-B^\gamma(x)$. 


Using the projections along the horospheres, we can define a new hyperbolic sliced-Wasserstein discrepancy, called horospherical, between $\mu, \nu \in \mathcal{P}_p(\mathbb{L}^d)$ as, for $p\ge 1$,
\begin{equation}
    HHSW_p^p(\mu,\nu) = \int_{T_{x^0}\mathbb{L}^d\cap S^d} W_p^p(B^v_\#\mu, B^v_\#\nu)\ \mathrm{d}\lambda(v).
\end{equation}
Note that $HHSW_p(\mu,\nu)<\infty$ for $\mu,\nu\in\mathcal{P}_p(\mathbb{L}^d)$ (see Appendix \ref{appendix:finiteness}). 
We also provide a formulation on the Poincaré ball between $\mu,\nu\in\mathcal{P}_p(\mathbb{B}^d)$, using $B^{\Tilde{v}}$, as
\begin{equation}
    HHSW_p^p(\mu,\nu) = \int_{S^{d-1}} W_p^p(B^{\Tilde{v}}_\#\mu, B^{\Tilde{v}}_\#\nu)\ \mathrm{d}\lambda(\Tilde{v}).
\end{equation}
Using that the projections formula between $\mathbb{L}^d$ and $\mathbb{B}^d$ are isometries, we show in the next proposition that the two formulations are equivalent. Hence, we choose in practice the formulation which is the more suitable, either from the nature of data or from a numerical stability viewpoint.

\begin{proposition} \label{prop:equality_hhsw}
    For $p\ge 1$, let $\mu,\nu\in\mathcal{P}_p(\mathbb{B}^d)$ and denote $\Tilde{\mu} = (P_{\mathbb{B}\to\mathbb{L}})_\#\mu$, $\Tilde{\nu} = (P_{\mathbb{B}\to\mathbb{L}})_\#\nu$. Then,
    \begin{align}
        HHSW_p^p(\mu,\nu) = HHSW_p^p(\Tilde{\mu},\Tilde{\nu}), \\
        GHSW_p^p(\mu,\nu) = GHSW_p^p(\Tilde{\mu}, \Tilde{\nu}).
    \end{align}
\end{proposition}

\subsection{Properties}

It can easily be showed that GHSW and HHSW are pseudo-distances as it only depends on the distance properties of the Wasserstein distance. Whether or not they satisfy the indiscernible property remains an open question. 
As described in the introduction for SW, we can derive the corresponding Radon transform. More precisely, we can show that 
\begin{equation}
    GHSW_p^p(\mu,\nu) = \int_{S^{d-1}} W_p^p\big((\Bar{R}\mu)^v, (\Bar{R}\nu)^v\big)\ \mathrm{d}\lambda(v),
\end{equation}
where $\Bar{R}$ is the hyperbolical Radon transform, first introduced by \citet{helgason1959differential} and more recently studied \emph{e.g.} in \citep{berenstein1999radon, berenstein2004totally, rubin2002radon}. We can also show a similar relation between HHSW and the horospherical Radon transform studied \emph{e.g.} by \citet{bray2019radon, casadio2021radon}. If these transforms are injective on the space of measures, then we would have that GHSW or HHSW are distances. However, to the best of our knowledge, the injectivity of such transforms on the space of measures has not been studied yet. We detail the derivations in Appendix \ref{appendix:pseudo_distance}.

We also provide in Appendix \ref{appendix:sample_complexity} the sample complexity and the projection complexity. We note that the results are similar as in the Euclidean case \citep{nadjahi2020statistical}, \emph{i.e.} the sample complexity is independent of the dimension and the projection complexity converges in $O(1/\sqrt{L})$ with $L$ the number of projections.

\section{Implementation}

\begin{algorithm}[tb]
   \caption{Guideline of GHSW}
   \label{alg:hsw}
    \begin{algorithmic}
       \STATE {\bfseries Input:} $(x_i)_{i=1}^n\sim \mu$, $(y_j)_{j=1}^n\sim \nu$, $(\alpha_i)_{i=1}^n$, $(\beta_j)_{j=1}^n\in \Delta_n$, $L$ the number of projections, $p$ the order
       \FOR{$\ell=1$ {\bfseries to} $L$}
       \STATE Draw $\Tilde{v}\sim\mathrm{Unif}(S^{d-1})$, let $v=[0,\Tilde{v}]$
       \STATE $\forall i,j,\ \hat{x}_i^{\ell}=P^v(x_i)$, $\hat{y}_j^\ell=P^v(y_j)$
       \STATE Compute $W_p^p(\sum_{i=1}^n \alpha_i \delta_{\hat{x}_i^\ell}, \sum_{j=1}^n \beta_j \delta_{\hat{y}_j^\ell})$
       \ENDFOR
       \STATE Return $\frac{1}{L}\sum_{\ell=1}^L W_p^p(\sum_{i=1}^n \alpha_i \delta_{\hat{x}_i^\ell}, \sum_{j=1}^n \beta_j \delta_{\hat{y}_j^\ell})$
    \end{algorithmic}
\end{algorithm}

\begin{figure}[t]
    \centering
    \includegraphics[width=\columnwidth]{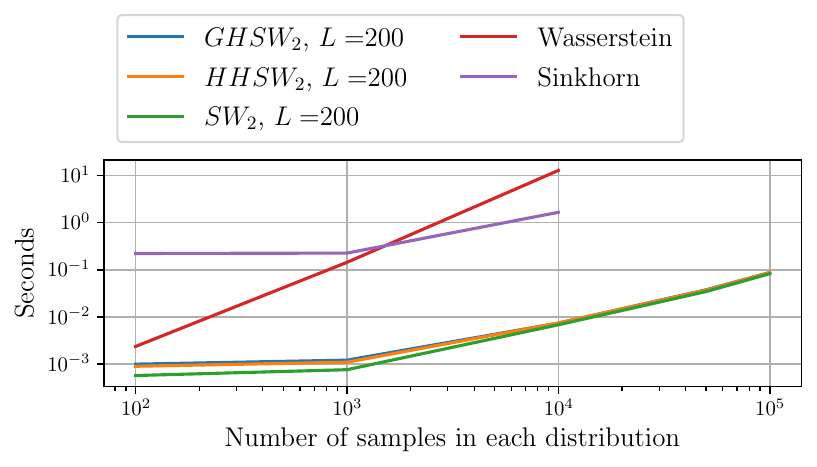}
    \caption{Runtime comparison in log-log scale between Wasserstein and Sinkhorn using the geodesic distance, $SW_2$, $GHSW_2$ and $HHSW_2$ with 200 projections, including the computation time of the cost matrices.}
    \label{fig:comparison_time}
    \vspace{-10pt}
\end{figure}

\begin{figure*}[t]
    \centering
    \hspace*{\fill}
    \subfloat[SW on Poincaré (SWp), GHSW]{\label{fig:w_wnd}\includegraphics[width=0.3\linewidth]{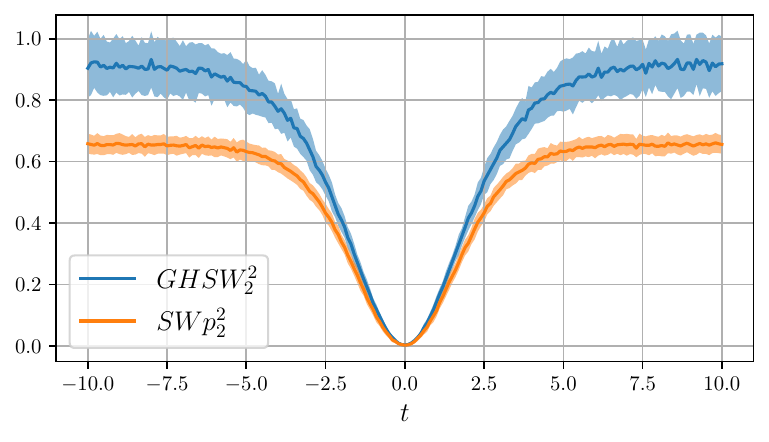}} \hfill
    \subfloat[HHSW and Wasserstein]{\label{fig:hsw_wnd}\includegraphics[width=0.3\linewidth]{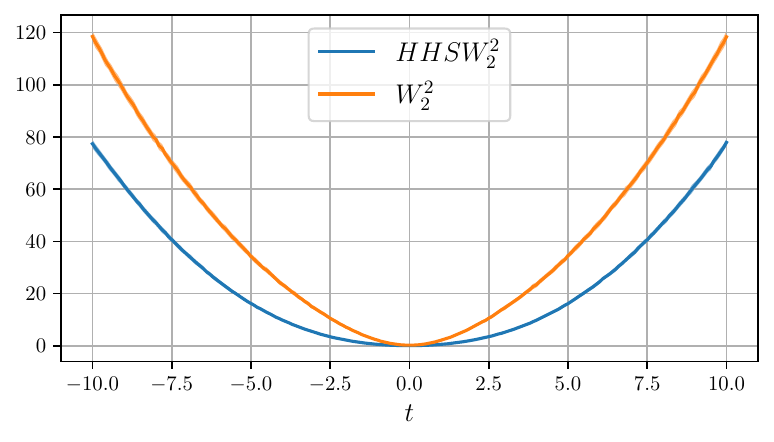}} \hfill
    \subfloat[SW on Lorentz (SWl)]{\label{fig:sw_wnd_lorentz}\includegraphics[width=0.3\linewidth]{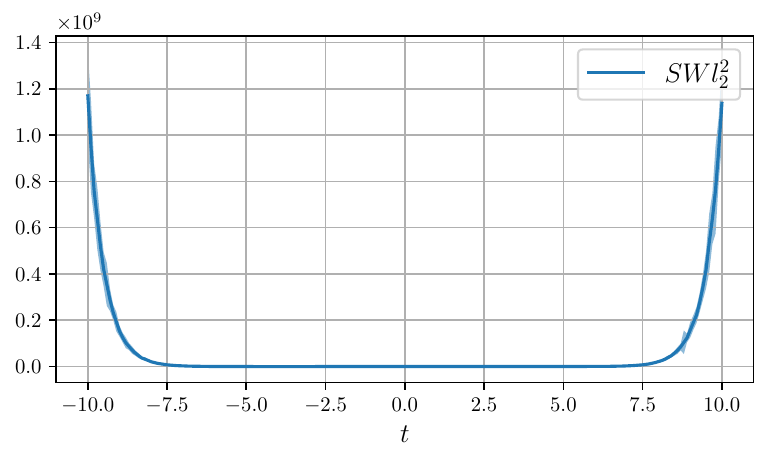}} \hfill
    \hspace*{\fill} \\
    \caption{Comparison of the Wasserstein distance (with the geodesic distance as cost), GHSW, HHSW and SW between Wrapped Normal distributions. We gather the discrepancies together by scale of the values. SW on the Poincaré model has very small values as it operates on the unit ball, while on the Lorentz model, it can take very high values. GHSW returns small values as the geodesic projections tend to project the points close to the origin. HHSW has values which are closer to the geodesic Wasserstein distance as the horospherical projection tends to better keep the distance between points.}
    \label{fig:comparison_wnd}
    \vspace{-15pt}
\end{figure*}

In this Section, we discuss the implementation of GHSW and HHSW, as well as their complexity.

\textbf{Implementation.} In practice, we only have access to discrete distributions $\hat{\mu}_n=\sum_{i=1}^n \alpha_i\delta_{x_i}$ and $\hat{\nu}_n = \sum_{i=1}^n \beta_i \delta_{y_i}$ where $(x_i)_i$ and $(y_i)_i$ are sample locations in hyperbolic space, and $(\alpha_i)_i$ and $(\beta_i)_i$ belong to the simplex $\Delta_n = \{\alpha\in [0,1]^n,\ \sum_{i=1}^n \alpha_i = 1\}$. We approximate the integral by a Monte-Carlo approximation by drawing a finite number $L$ of projection directions $(v_\ell)_{\ell=1}^L$ in $S^{d-1}$. Then, computing GHSW and HHSW amount at first getting the coordinates on $\mathbb{R}$ by using the corresponding projections, and computing the 1D Wasserstein distance between them. We summarize the procedure in Algorithm \ref{alg:hsw} for GHSW. 

\looseness=-1 \textbf{Complexity.} For both GHSW and HHSW, the projection procedure has a complexity of $O(nd)$. Hence, for $L$ projections, the complexity is in $O(Ln(d+\log n))$ which is the same as for SW. In Figure \ref{fig:comparison_time}, we compare the runtime between GHSW, HHSW, SW, Wasserstein and Sinkhorn with geodesic distances in $\mathbb{L}^2$ for $n\in \{10^2,10^3,10^4,5\cdot 10^4, 10^5\}$ samples which are drawn from wrapped normal distributions \citep{nagano2019wrapped}, and $L=200$ projections. We used the POT library \citep{flamary2021pot} to compute SW, Wasserstein and Sinkhorn. We observe  the quasi-linearity complexity of GHSW and HHSW. When we only have a few samples, the cost of the projection is higher than computing the 1D Wasserstein distance, and SW is the fastest. 

\section{Application}

In this Section, we perform several experiments which aim at comparing GHSW, HHSW, SWp and SWl. First, we 
study the evolution of the different distances between wrapped normal distributions which move along geodesics. Then, we illustrate the ability to fit distributions on $\mathbb{L}^2$ using gradient flows. Finally, we use HHSW and GHSW for an image classification problem where they are used to fit a prior in the embedding space. We add more informations about distributions and optimization in hyperbolic spaces in Appendix \ref{appendix:hyperbolic_space}. Complete details of the experimental settings are reported in Appendix \ref{appendix:xps}. We also report in \cref{appendix:hswae} preliminary experiments on autoencoders with hierarchical latent priors.

\paragraph{Comparisons of the Different Hyperbolical SW Discrepancies.} \label{section:comparison_sw}

On Figure \ref{fig:comparison_wnd}, we compare the evolutions of GHSW, HHSW, SW and Wasserstein with the geodesic distance between Wrapped Normal Distributions (WNDs), where one is centered and the other moves along a geodesic. More precisely, by denoting $\mathcal{G}(\mu,\Sigma)$ a WND, we plot the evolution of the distances between $\mathcal{G}(x^0,I_2)$ and $\mathcal{G}(x_t, I_2)$ where $x_t = \cosh(t)x^0 + \sinh(t)v$ for $t\in [-10,10]$ and $v\in T_{x^0}\mathbb{L}^2\cap S^2$. We observe first that SW on the Lorentz model explodes when the two distributions are getting far from each other. Then, we observe that $HHSW_2$ has values with a scale similar to $W_2$. We argue that it comes from the observation of \citet{chami2021horopca} which stated that the horospherical projection better preserves the distance between points compared to the geodesic projection. As SWp operates on the unit ball using Euclidean distances, the distances are very small, even for distributions close to the border. Interestingly, as geodesic projections tend to project points close to the origin, GHSW tends also to squeeze the distance between distributions far from the origin. This might reduce numerical instabilities when getting far from the origin, especially in the Lorentz model. This experiments also allows to observe that, at least for
WNDs, the indiscernible property is observed in practice as we only obtain one minimum when both measures coincide. Hence, it suggests that GHSW and HHSW are proper distances.

\paragraph{Gradient Flows.}

\begin{figure}[t]
    \centering
    \includegraphics[width=\columnwidth]{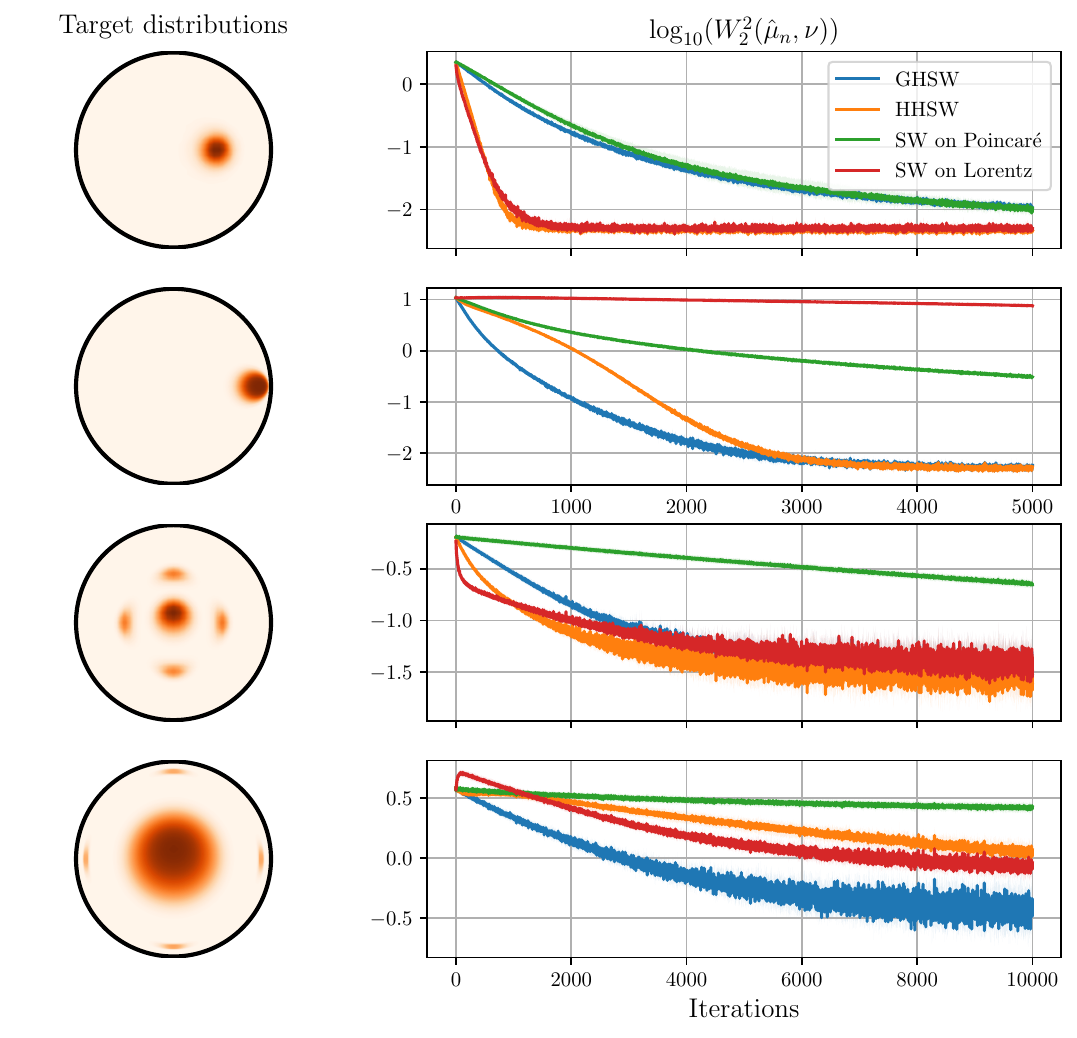}
    \caption{Log 2-Wasserstein between a target and the gradient flow of GHSW, HHSW and SW (averaged over 5 runs).}
    \label{fig:comparison_gradientflows}
    \vspace{-15pt}
\end{figure}

\looseness=-1 We now assess the ability to learn distributions by minimizing the hyperbolic SW discrepancies ($HSW$). We suppose that we have a target distribution $\nu$ from which we have access to samples $(x_i)_{i=1}^n$. Therefore, we aim at learning $\nu$ by solving the following optimization problem: $\min_{\mu} \ HSW_2^2\big(\mu,\frac{1}{n}\sum_{i=1}^n \delta_{x_i}\big)$. We model $\mu$ as a set of $n=500$ particles and propose to perform a Riemannian gradient descent \citep{boumal2022intromanifolds} to learn the distribution. 

To compare the dynamics of the different discrepancies, we plot on Figure \ref{fig:comparison_gradientflows} the evolution of the exact log $2$-Wasserstein distance, with geodesic distance as ground cost, between the learned distribution at each iteration and the target, with the same learning rate. We use as targets wrapped normal distributions and mixtures of WNDs. For each type of target, we consider two settings, one in which the distribution is close to the origin and another in which the distribution lies closer to the border. We observe different behaviors in the two settings. When the target is lying close to the origin, SWl and HHSW, which present the biggest magnitude, are the fastest to converge. As for distant distributions however, GHSW converges the fastest. Moreover, SWl suffers from many numerical instabilities, as the projections of the gradients do not necessarily lie on the tangent space when points are too far of the origin. This requires to lower the learning rate, and hence to slow down the convergence. Interestingly, SWp is the slowest to converge in both settings.

\paragraph{Deep Classification with Prototypes.}

We now turn to a classification use case with real world data. Let $\{(x_i,y_i)_{i=1}^n\}$ be a training set where $x_i\in\mathbb{R}^m$ and $y_i\in\{1,\dots,C\}$ denotes a label. \citet{ghadimi2021hyperbolic} perform classification on the Poincaré ball by assigning to each class $c\in\{1,\dots,C\}$ a prototype $p_c\in S^{d-1}$, and then by learning an embedding on the hyperbolic space using a neural network $f_\theta$ followed by the exponential map. Then, by denoting by $z=\exp_0\big(f_\theta(x)\big)$ the output, the loss to be minimized is, for a regularization parameter $s\ge 0$,
\begin{equation}
    \ell(\theta) = \frac{1}{n}\sum_{i=1}^n \Big(B^{p_{y_i}}\big(z_i\big) - sd\cdot \log\big(1-\|z_i\|_2^2\big)\Big).
\end{equation}
The first term is the Busemann function which will draw the representations of $x_i$ towards the prototype assigned to the class $y_i$, while the second term penalizes the overconfidence and pulls back the representation towards the origin. \citet{ghadimi2021hyperbolic} showed that the second term can be decisive to improve the accuracy. Then, the classification of an input is done by solving $y^* = \argmax_{c}\ \langle \frac{z}{\|z\|}, p_c\rangle$.

\looseness=-1 We propose to replace the second term by a global prior on the distribution of the representations. More precisely, we add a discrepancy $D$ between the distribution $(\exp_0\circ f_\theta)_\# p_X$, where $p_X$ denotes the distribution of the training set, and a mixture of $C$ WNDs where the centers are chosen as $(\alpha p_c)_{c=1}^C$, with $(p_c)_c$ the prototypes and $0<\alpha<1$. In practice, we use $D=GHSW_2^2$, $D=HHSW_2^2$, $D=SWp_2^2$ and $D=SWl_2^2$ to assess their usability on a real problem and compared with $W_2^2$ and MMD with Laplacian kernel \citep{feragen2015geodesic}. Let $(w_i)_{i=1}^n$ be a batch of points drawn from this mixture, then the loss we minimize is
\begin{equation} \label{eq:loss_hsw}
    \ell(\theta) = \frac{1}{n}\sum_{i=1}^n B^{p_i}  (z_i)  + \lambda D\left(\frac{1}{n}\sum_{i=1}^n \delta_{z_i}, \frac{1}{n}\sum_{i=1}^n \delta_{w_i}\right).
\end{equation}
On Table \ref{tab:acc_pebuse}, we report the classification accuracy on the test set for CIFAR10 and CIFAR100 \citep{krizhevsky2009learning}, using the exact same setting as \citep{ghadimi2021hyperbolic}. We rerun their method, called PeBuse here. We report results averaged over 3 runs. 
We observe that the proposed penalization outperforms the original method for all the different dimensions.


\begin{table}[t]
    \centering
    \caption{Test Accuracy on deep classification with prototypes (best performance in bold)}
    \small
    \resizebox{\columnwidth}{!}{
        \begin{tabular}{ccccccc}
            & \multicolumn{2}{c}{CIFAR10} & & \multicolumn{3}{c}{CIFAR100} \\ 
            \toprule
            Dimensions & 2 & 4 & & 3 & 5 & 10 \\ \midrule
            PeBuse & $90.64_{\pm 0.06}$ & $90.59_{\pm 0.11}$ & & $49.28_{\pm 1.95}$ & $53.44_{\pm 0.76}$ & $59.19_{\pm 0.39}$ \\
            GHSW & $91.39_{\pm 0.23}$ & $91.66_{\pm 0.27}$ & & $\bold{53.97}_{\pm 1.35}$ & $60.64_{\pm 0.87}$ & $61.45_{\pm 0.41}$ \\
            HHSW & $91.28_{\pm 0.26}$ & $\bold{91.98}_{\pm 0.05}$ & & $53.88_{\pm 0.06}$ & $\bold{60.69}_{\pm 0.25}$ & $\bold{62.80}_{\pm 0.09}$ \\
            SWp & $\bold{91.84}_{\pm 0.31}$ & $91.68_{\pm 0.10}$ & & $53.25_{\pm 3.27}$ & $59.77_{\pm 0.81}$ & $60.36_{\pm 1.26}$ \\
            SWl & $91.13_{\pm 0.14}$ & $91.74_{\pm 0.12}$ & & $53.88_{\pm 0.02}$ & $60.62_{\pm 0.39}$ & $62.30_{\pm 0.23}$ \\
            W & $91.67_{\pm 0.18}$ & $91.83_{\pm 0.21}$ & & $50.07_{\pm 4.58}$ & $57.49_{\pm 0.94}$ & $58.82_{\pm 1.66}$ \\
            MMD & $91.47_{\pm 0.10}$ & $91.68_{\pm 0.09}$ & & $50.59_{\pm 4.44}$ & $58.10_{\pm 0.73}$ & $58.91_{\pm 0.91}$ \\
            \bottomrule
        \end{tabular}
    }
    \vspace{-15pt}
    \label{tab:acc_pebuse}
\end{table}

\section{Conclusion and Discussion}

\looseness=-1 In this work, we propose different sliced-Wasserstein discrepancies between distributions lying in hyperbolic spaces. In particular, we introduce two new SW discrepancies which are intrinsically defined on hyperbolic spaces. They are built by first identifying a closed-form for the Wasserstein distance on geodesics, and then by using different projections on the geodesics. We compare these metrics on multiple tasks such as sampling and image classification. We observe that, while Euclidean SW in the ambient space still works, it suffers from either slow convergence on the Poincaré ball or numerical instabilities on the Lorentz model when distributions are lying far from the origin. On the other hand, geodesic versions exhibit the same complexity and converge generally better for gradient flows. Further works will look into other tasks where hyperbolic embeddings and distributions have been showed to be beneficial, such as persistent diagrams \citep{carriere2017sliced, kyriakis2021learning}. Besides further applications, proving that these discrepancies are indeed distances, and deriving statistical results are interesting directions of work. One might also consider different subspaces on which to project, such as horocycles which are circles of infinite radius and which can be seen as another analog object to lines in hyperbolic spaces \citep{casadio2021radon}.

\section*{Acknowledgements}

This research was funded by project DynaLearn from Labex CominLabs and Region Bretagne ARED DLearnMe, and by the project OTTOPIA ANR-20-CHIA-0030 of the French National Research Agency (ANR).



\bibliography{references}
\bibliographystyle{icml2023}


\newpage
\appendix
\onecolumn

\section{Proofs} \label{appendix:proofs}

\subsection{Proof of Proposition \ref{prop:wasserstein_geodesics}} \label{proof:prop_wasserstein_geodesics}

Let $\gamma$ a geodesic on $\mathbb{L}^d$ passing through $x^0$ and with direction $v\in T_{x^0}\mathbb{L}^d\cap S^d$, \emph{i.e.} the geodesic is obtained as $\mathrm{span}(x^0,v)\cap\mathbb{L}^d$. Let $\mu,\nu$ probability measures on $\gamma$.

First, we need to show that for all $x,y\in\mathrm{span}(x^0,v)\cap\mathbb{L}^d$,
\begin{equation}
    d_\mathbb{L}(x,y) = |t^v(x)-t^v(y)|,
\end{equation}
\emph{i.e.} that $t^v$ is an isometry from $\mathrm{span}(x^0,v)\cap\mathbb{L}^d$ to $\mathbb{R}$.

As $x$ and $y$ belong to the geodesic $\gamma$, there exist $s,t\in\mathbb{R}$ such that 
\begin{equation}
    x = \gamma(s) = \cosh(s) x^0 + \sinh(s)v,
\end{equation}
and
\begin{equation}
    y = \gamma(t) = \cosh(t) x^0 + \sinh(t)v.
\end{equation}
Then, on one hand, we have
\begin{equation}
    \begin{aligned}
        d_\mathbb{L}(\gamma(s),\gamma(t)) &= \arccosh(-\langle \gamma(s),\gamma(t)\rangle_\mathbb{L}) \\
        &= \arccosh\big(-\langle \cosh(t) x^0 + \sinh(t) v, \cosh(s)x^0 + \sinh(s) v\rangle\big) \\
        &= \arccosh\big(\cosh(t)\cosh(s) - \sinh(t)\sinh(s)\big) \\
        &= \arccosh\big(\cosh(t-s)\big) \\
        &= |t-s|,
    \end{aligned}
\end{equation}
where we used that $\langle x^0,x^0\rangle_\mathbb{L}=-1$, $\langle x^0,v\rangle_\mathbb{L}=0$, $\langle v,v\rangle_\mathbb{L}=\langle v,v\rangle = 1$ and $\cosh(t)\cosh(s)-\sinh(t)\sinh(s)=\cosh(t-s)$.

On the other hand, we have
\begin{equation}
    \begin{aligned}
        |t^v(x)-t^v(y)| &= \big|\mathrm{sign}(\langle x,v\rangle) d_\mathbb{L}(x,x^0) - \mathrm{sign}(\langle y,v\rangle) d_\mathbb{L}(y,x^0) \big| \\
        &= \big| \mathrm{sign}(\langle x,v\rangle) d_\mathbb{L}(\gamma(s),\gamma(0)) - \mathrm{sign}(\langle y,v\rangle) d_\mathbb{L}(\gamma(t),\gamma(0))\big| \\
        &= \big| \mathrm{sign}(\langle x,v\rangle) |s| - \mathrm{sign}(\langle y,v\rangle) |t|\big| \\
        &= |t-s|,
    \end{aligned}
\end{equation}
where we use at the last line that $\mathrm{sign}(\langle x,v\rangle) = \mathrm{sign}(s)$ (resp. $\mathrm{sign}(\langle y,v\rangle)=\mathrm{sign}(t)$) and $s=\mathrm{sign}(s) |s|$ (resp. $t=\mathrm{sign}(t)|t|$) supposing that $v$ is oriented in the same sense of $\gamma$.

Therefore, we have
\begin{equation}
    |t^v(x)-t^v(y)| = d_\mathbb{L}(x,y).
\end{equation}

Now, we can show the equality for the Wasserstein distance:
\begin{equation}
    \begin{aligned}
        W_p^p(\mu, \nu) &= \inf_{\gamma \in \Pi(\mu,\nu)}\ \int_{\mathbb{L}^d\times \mathbb{L}^d} d_\mathbb{L}(x,y)^p\ \mathrm{d}\gamma(x,y) \\
        &= \inf_{\gamma\in\Pi(\mu,\nu)}\ \int_{\mathbb{L}^d\times \mathbb{L}^d} |t^v(x)-t^v(y)|^p \ \mathrm{d}\gamma(x,y) \\
        &= \inf_{\gamma\in\Pi(\mu,\nu)}\ \int_{\mathbb{R}\times \mathbb{R}} |x-y|^p\ \mathrm{d}(t^v\otimes t^v)_\#\gamma(x,y) \\
        &= \inf_{\Tilde{\gamma} \in \Pi(t^v_\#\mu, t^v_\#\nu)}\ \int_{\mathbb{R}\times \mathbb{R}} |x-y|^p \ \mathrm{d}\Tilde{\gamma}(x,y) \\
        &= W_p^p(t^v_\#\mu, t^v_\#\nu) = \int_0^1 |F_{t^v_\#\mu}^{-1}(u)-F_{t^v_\#\nu}^{-1}(u)|^p\ \mathrm{d}u,
    \end{aligned}
\end{equation}
where we apply \citep[Lemma 6]{paty2019subspace}.

\subsection{Geodesic Projection} \label{appendix:geod_proj}

\begin{proposition}[Geodesic projection] \label{prop:hsw_geodesic_proj} \leavevmode
    \begin{enumerate}
        \item Let $\mathcal{G}^v=\mathrm{span}(x^0,v)\cap \mathbb{L}^d$ where $v\in T_{x^0}\mathbb{L}^d\cap S^d$. Then, the geodesic projection $\Tilde{P}^v$ on $\mathcal{G}^v$ of $x\in\mathbb{L}^d$ is
        \begin{equation}
            \begin{aligned}
                \Tilde{P}^v(x) &= \frac{1}{\sqrt{\langle x,x^0\rangle_\mathbb{L}^2-\langle x, v\rangle_\mathbb{L}^2}} \big(-\langle x,x^0\rangle_\mathbb{L}x^0 + \langle x,v\rangle_\mathbb{L} v\big) \\
                &= \frac{P^{\mathrm{span}(x^0, v)}(x)}{\sqrt{-\langle P^{\mathrm{span}(x^0, v)}(x), P^{\mathrm{span}(x^0, v)}(x)\rangle_\mathbb{L}}},
            \end{aligned}
        \end{equation}
        where $P^{\mathrm{span}(x^0,v^0)}$ is the linear orthogonal projection on the subspace $\mathrm{span}(x^0,v)$.
        \item Let $\Tilde{v}\in S^{d-1}$ be an in ideal point. Then, the geodesic projection $\Tilde{P}^{\Tilde{v}}$ on the geodesic characterized by $\Tilde{v}$ of $x\in \mathbb{B}^d$ is
        \begin{equation}
            \Tilde{P}^{\Tilde{v}}(x) = s(x) \Tilde{v},
        \end{equation}
        where
        \begin{equation}
            s(x) = \left\{\begin{array}{ll} \frac{1+\|x\|_2^2 - \sqrt{(1+\|x\|_2^2)^2 - 4 \langle x, \Tilde{v}\rangle^2}}{2 \langle x, \Tilde{v}\rangle} & \mbox{ if } \langle x,\Tilde{v}\rangle \neq 0 \\
            0 & \mbox{ if } \langle x,\Tilde{v}\rangle = 0.
            \end{array}\right.
        \end{equation}
    \end{enumerate}
\end{proposition}

\begin{proof} \label{proof:geodesic_proj} \leavevmode
    \begin{enumerate}
        \item \textbf{Lorentz model.} Any point $y$ on the geodesic obtained by the intersection between $E=\mathrm{span}(x^0,v)$ and $\mathbb{L}^d$ can be written as
        \begin{equation}
            y = \cosh(t) x^0 + \sinh(t) v,
        \end{equation}
        where $t\in\mathbb{R}$. Moreover, as $\arccosh$ is an increasing function, we have
        \begin{equation}
            \begin{aligned}
                \Tilde{P}^v(x) &= \argmin_{y\in E\cap\mathbb{L}^d}\ d_\mathbb{L}(x,y) \\
                &= \argmin_{y\in E\cap \mathbb{L}^d}\ -\langle x,y\rangle_\mathbb{L}.
            \end{aligned}
        \end{equation}
        This problem is equivalent with solving
        \begin{equation}
            \argmin_{t\in\mathbb{R}}\ -\cosh(t)\langle x,x^0\rangle_\mathbb{L} - \sinh(t)\langle x,v\rangle_\mathbb{L}.
        \end{equation}
        Let $g(t)=-\cosh(t)\langle x,x^0\rangle_\mathbb{L} - \sinh(t)\langle x,v\rangle_\mathbb{L}$, then
        \begin{equation} \label{eq:hsw_coord_opt}
            g'(t) = 0 \iff \tanh(t) = - \frac{\langle x,v\rangle_\mathbb{L}}{\langle x,x^0\rangle_\mathbb{L}}.
        \end{equation}
        Finally, using that $1-\tanh^2(t)=\frac{1}{\cosh^2(t)}$ and $\cosh^2(t)-\sinh^2(t)=1$, and observing that necessarily, $\langle x,x^0\rangle_\mathbb{L} \le 0$, we obtain
        \begin{equation}
            \cosh(t) = \frac{1}{\sqrt{1-\left(-\frac{\langle x,v\rangle_\mathbb{L}}{\langle x,x^0\rangle_\mathbb{L}}\right)^2}} = \frac{-\langle x,x^0\rangle_\mathbb{L}}{\sqrt{\langle x,x^0\rangle_\mathbb{L}^2 - \langle x,v\rangle_\mathbb{L}^2}},
        \end{equation}
        and
        \begin{equation}
            \sinh(t) = \frac{-\frac{\langle x,v\rangle_\mathbb{L}}{\langle x,x^0\rangle_\mathbb{L}}}{\sqrt{1-\left(-\frac{\langle x,v\rangle_\mathbb{L}}{\langle x,x^0\rangle_\mathbb{L}}\right)^2}} = \frac{\langle x, v\rangle_\mathbb{L}}{\sqrt{\langle x,x^0\rangle_\mathbb{L}^2 - \langle x,v\rangle_\mathbb{L}^2}}.
        \end{equation}
        \item \textbf{Poincaré ball.} A geodesic passing through the origin on the Poincaré ball is of the form $\gamma(t) = tp$ for an ideal point $p\in S^{d-1}$ and $t\in ]-1,1[$. Using that $\arccosh$ is an increasing function, we find
        \begin{equation}
            \begin{aligned}
                \Tilde{P}^p(x) &= \argmin_{y \in \mathrm{span}(\gamma)}\ d_\mathbb{B}(x,y) \\
                &= \argmin_{tp}\ \arccosh\left(1 + 2 \frac{\|x-\gamma(t)\|_2^2}{(1-\|x\|_2^2)(1-\|\gamma(t)\|_2^2)}\right) \\
                &= \argmin_{tp}\ \log\big(\|x-\gamma(t)\|_2^2\big) - \log\big(1-\|x\|_2^2\big) - \log\big(1-\|\gamma(t)\|_2^2\big) \\
                &= \argmin_{tp}\ \log\big(\|x-tp\|_2^2\big) - \log\big(1-t^2\big).
            \end{aligned}
        \end{equation}
        Let $g(t) = \log\big(\|x-tp\|_2^2\big) - \log\big(1-t^2\big)$. Then,
        \begin{equation}
            g'(t) = 0 \iff  \left\{\begin{array}{ll}
          t^2 - \frac{1+\|x\|_2^2}{\langle x,p\rangle} t + 1 = 0 & \mbox{ if } \langle p,x\rangle \neq 0, \\
          t = 0 & \mbox{ if } \langle p,x\rangle = 0.
          \end{array}\right.
        \end{equation}
        Finally, if $\langle x, p\rangle \neq 0$, the solution is
        \begin{equation}
            t = \frac{1+\|x\|_2^2}{2\langle x,p\rangle} \pm \sqrt{\left(\frac{1+\|x\|_2^2}{2\langle x,p\rangle}\right)^2 - 1}.
        \end{equation}

        Now, let us suppose that $\langle x, p\rangle > 0$. Then, 
        \begin{equation}
            \begin{aligned}
                \frac{1+\|x\|_2^2}{2\langle x,p\rangle} + \sqrt{\left(\frac{1+\|x\|_2^2}{2\langle x,p\rangle}\right)^2 - 1} &\ge \frac{1+\|x\|_2^2}{2\langle x,p\rangle} \\
                & \ge 1,
            \end{aligned}
        \end{equation}
        because $\|x-p\|_2^2 \ge 0$ implies that $\frac{1+\|x\|_2^2}{2\langle x,p\rangle} \ge 1$, and therefor the solution is
        \begin{equation}
            t = \frac{1+\|x\|_2^2}{2\langle x,p\rangle} - \sqrt{\left(\frac{1+\|x\|_2^2}{2\langle x,p\rangle}\right)^2 - 1}.
        \end{equation}

        Similarly, if $\langle x, p\rangle < 0$, then
        \begin{equation}
            \begin{aligned}
                \frac{1+\|x\|_2^2}{2\langle x,p\rangle} - \sqrt{\left(\frac{1+\|x\|_2^2}{2\langle x,p\rangle}\right)^2 - 1} &\le  \frac{1+\|x\|_2^2}{2\langle x,p\rangle} \\
                &\le -1,
            \end{aligned}
        \end{equation}
        because $\|x+p\|_2^2 \ge 0$ implies $\frac{1+\|x\|_2^2}{2\langle x,p\rangle} \le -1$, and the solution is
        \begin{equation}
            \frac{1+\|x\|_2^2}{2\langle x,p\rangle} + \sqrt{\left(\frac{1+\|x\|_2^2}{2\langle x,p\rangle}\right)^2 - 1}.
        \end{equation}
        Thus,
        \begin{equation}
            \begin{aligned}
                s(x) &= \begin{cases}
                    \frac{1+\|x\|_2^2}{2\langle x, \Tilde{v}\rangle} - \sqrt{\left(\frac{1+\|x\|_2^2}{2\langle x,\Tilde{v}\rangle}\right)^2 -1} \quad \text{if } \langle x, \Tilde{v}\rangle > 0 \\
                    \frac{1+\|x\|_2^2}{2\langle x, \Tilde{v}\rangle} + \sqrt{\left(\frac{1+\|x\|_2^2}{2\langle x,\Tilde{v}\rangle}\right)^2 -1} \quad \text{if } \langle x, \Tilde{v}\rangle < 0.
                \end{cases} \\
                &= \frac{1+\|x\|_2^2}{2\langle x, \Tilde{v}\rangle} - \mathrm{sign}(\langle x,\Tilde{v}\rangle \sqrt{\left(\frac{1+\|x\|_2^2}{2\langle x,\Tilde{v}\rangle}\right)^2 -1} \\
                &= \frac{1+\|x\|_2^2}{2\langle x, \Tilde{v}\rangle}  - \frac{\mathrm{sign}(\langle x,\Tilde{v}\rangle}{2\mathrm{sign}(\langle x,\Tilde{v}\rangle\langle x, \Tilde{v}\rangle} \sqrt{(1+\|x\|_2^2)^2 - 4\langle x,\Tilde{v}\rangle^2} \\
                &= \frac{1+\|x\|_2^2 - \sqrt{(1+\|x\|_2^2)^2 - 4 \langle x, \Tilde{v}\rangle^2}}{2 \langle x, \Tilde{v}\rangle}.
            \end{aligned}
        \end{equation}
    \end{enumerate}
\end{proof}

We observe that the projection on the geodesic in the Lorentz model can be done by first projecting on the subspace $\mathrm{span}(x^0,v)$ and then by projecting on the hyperboloid by normalizing. This is 
analogous to the spherical case studied in \citep{bonet2022spherical}, 
the differences being that, in the hyperbolic case, we are on the Minkowski space and that the geodesics are not periodic, contrary to the sphere. Moreover, we only integrate \emph{w.r.t.} geodesics passing through the origin when \citet{bonet2022spherical} integrate over all possible geodesics, as the sphere does not have a natural origin.

\subsection{Proof of Proposition \ref{prop:hsw_coord_geod_proj}} \label{proof:hsw_coord_geod_proj}

\begin{enumerate}
    \item \textbf{Lorentz model.} The coordinate on the geodesic can be obtained as 
    \begin{equation}
        P^v(x) = \argmin_{t\in\mathbb{R}}\ d_\mathbb{L}\big(\exp_{x^0}(tv), x\big).
    \end{equation}
    Hence, by using \eqref{eq:hsw_coord_opt}, we obtain that the optimal $t$ satisfies
    \begin{equation}
        \tanh(t) = - \frac{\langle x, v\rangle_\mathbb{L}}{\langle x,x^0\rangle_\mathbb{L}} \iff t = \arctanh\left(-\frac{\langle x, v\rangle_\mathbb{L}}{\langle x, x^0\rangle_\mathbb{L}}\right).
    \end{equation}
    \item \textbf{Poincaré ball.} As a geodesic is of the form $\gamma(t)=\tanh\left(\frac{t}{2}\right) p$ for all $t\in\mathbb{R}$, we deduce from \Cref{prop:hsw_geodesic_proj} that
    \begin{equation}
        s(x) = \tanh\left(\frac{t}{2}\right) \iff t = 2 \arctanh\big(s(x)\big).
    \end{equation}
\end{enumerate}

\subsection{Proof of Proposition \ref{prop:busemann_closed_forms}} \label{proof:busemann_closed_forms}

\begin{enumerate}
    \item \textbf{Lorentz model.}
    
        The geodesic in direction $v$ can be characterized by
        \begin{equation}
            \forall t\in \mathbb{R},\ \gamma_v(t) = \cosh(t)x^0 + \sinh(t) v.
        \end{equation}
        Hence, we have
        \begin{equation}
            \begin{aligned}
                \forall x\in \mathbb{L}^d,\ d_{\mathbb{L}}(\gamma_v(t), x) &= \arccosh(-\cosh(t)\langle x,x^0\rangle_\mathbb{L} - \sinh(t) \langle x, v\rangle_\mathbb{L} ) \\
                &= \arccosh\left(-\frac{e^t + e^{-t}}{2} \langle x,x^0\rangle_\mathbb{L} - \frac{e^t - e^{-t}}{2}\langle x,v\rangle_\mathbb{L}\right) \\
                &= \arccosh\left(\frac{e^t}{2}\big((-1-e^{-2t})\langle x, x^0\rangle_\mathbb{L} + (-1+e^{-2t})\langle x,v\rangle_\mathbb{L}\big)\right) \\
                &= \arccosh\big(x(t)\big).
            \end{aligned}
        \end{equation}
        Then, on one hand, we have $x(t) \underset{t\to\infty}{\to} \pm \infty$, and using that $\arccosh(x)=\log\big(x+\sqrt{x^2-1}\big)$, we have
        \begin{equation}
            \begin{aligned}
                d_{\mathbb{L}}(\gamma_v(y),x) - t &= \log\left(\big(x(t) + \sqrt{x(t)^2-1}\big) e^{-t}\right) \\
                &= \log\left(e^{-t} x(t) + e^{-t}x(t) \sqrt{1-\frac{1}{x(t)^2}}\right) \\
                &\underset{\infty}{=} \log\left(e^{-t}x(t) + e^{-t}x(t) \left(1-\frac{1}{2x(t)^2} + o\left(\frac{1}{x(t)^2}\right)\right)\right).
            \end{aligned}
        \end{equation}
        Moreover,
        \begin{equation}
            e^{-t}x(t) = \frac12 (-1-e^{-2t})\langle x,x^0\rangle_\mathbb{L} + \frac12(-1+e^{-2t})\langle x,v\rangle_\mathbb{L} \underset{t\to\infty}{\to} -\frac12 \langle x,  x^0 + v\rangle_\mathbb{L}.
        \end{equation}
        Hence,
        \begin{equation}
            B^v(x) = \log(-\langle x, x^0+v\rangle_\mathbb{L}).
        \end{equation}
        
    \item \textbf{Poincaré ball.}
    
    Note that this proof can be found \emph{e.g.} in the Appendix of \citep{ghadimi2021hyperbolic}. We report it for the sake of completeness.
    
        Let $p\in S^{d-1}$, then the geodesic from $0$ to $p$ is of the form $\gamma_p(t) = \exp_0(tp) = \tanh(\frac{t}{2})p$.
        Moreover, recall that $\arccosh(x) = \log(x+\sqrt{x^2 -1})$ and 
        \begin{equation}
            d_{\mathbb{B}}(\gamma_p(t),x) = \arccosh\left(1+2\frac{\|\tanh(\frac{t}{2})p-x\|_2^2}{(1-\tanh^2(\frac{t}{2}))(1-\|x\|_2^2)}\right) = \arccosh(1+x(t)),
        \end{equation}
        where
        \begin{equation}
            x(t) = 2\frac{\|\tanh(\frac{t}{2})p-x\|_2^2}{(1-\tanh^2(\frac{t}{2}))(1-\|x\|_2^2)}.
        \end{equation}
        Now, on one hand, we have
        \begin{equation}
            \begin{aligned}
                B^p(x) &= \lim_{t\to\infty}\ (d_\mathbb{B}(\gamma_p(t),x)-t) \\
                &= \lim_{t\to \infty}\ \log\big(1+x(t)+\sqrt{x(t)^2+2x(t)}\big)-t \\
                &= \lim_{t\to \infty}\ \log\big(e^{-t}(1+x(t)+\sqrt{x(t)^2 + 2x(t)})\big).
            \end{aligned}
        \end{equation}
        On the other hand, using that $\tanh(\frac{t}{2}) = \frac{e^t-1}{e^t+1}$,
        \begin{equation}
            \begin{aligned}
                e^{-t}x(t) &= 2e^{-t} \frac{\|\frac{e^t-1}{e^t+1}p-x\|_2^2}{(1-(\frac{e^t-1}{e^t+1})^2)(1-\|x\|_2^2)} \\
                &= 2e^{-t} \frac{\|e^t p - p - e^t x - x\|_2^2}{4e^t (1-\|x\|_2^2)} \\
                &= \frac12 \frac{\|p-e^{-t}p-x-e^{-t}x\|_2^2}{1-\|x\|_2^2} \\
                &\underset{t\to\infty}{\to}\frac12 \frac{\|p-x\|_2^2}{1-\|x\|_2^2}.
            \end{aligned}
        \end{equation}
        Hence,
        \begin{equation}
            \begin{aligned}
                B^p(x) &= \lim_{t\to\infty}\ \log\left( e^{-t} + e^{-t}x(t) + e^{-t}x(t)\sqrt{1+\frac{2}{x(t)}}\right) = \log\left(\frac{\|p-x\|_2^2}{1-\|x\|_2^2}\right),
            \end{aligned}
        \end{equation}
        using that $\sqrt{1+\frac{2}{x(t)}} = 1 + \frac{1}{x(t)} + o(\frac{1}{x(t)})$ and $\frac{1}{x(t)}\to_{t\to\infty} 0$.
\end{enumerate}

\subsection{Horospherical Projections} \label{appendix:busemann_closed_forms}

\begin{proposition}[Horospherical projection] \label{prop:horospherical_projection} \leavevmode
    \vspace{-1em}
    \begin{enumerate}
        \item Let $v\in T_{x^0}\mathbb{L}^d\cap S^d$ be a direction and $\mathcal{G}=\mathrm{span}(x^0,v)\cap \mathbb{L}^d$ the corresponding geodesic passing through $x^0$. Then, for any $x\in\mathbb{L}^d$, the projection on $\mathcal{G}$ along the horosphere is given by
        \begin{equation}
            \Tilde{B}^v(x) = \frac{1+u^2}{1-u^2} x^0 + \frac{2u}{1-u^2} v,
        \end{equation}
        where $u = \frac{1+\langle x, x^0+v\rangle_\mathbb{L}}{1-\langle x, x^0+v\rangle_\mathbb{L}}$.
        \item Let $\Tilde{v}\in S^{d-1}$ be an ideal point. Then, for all $x\in \mathbb{B}^d$,
        \begin{equation}
            \Tilde{B}^{\Tilde{v}}(x) = \left(\frac{1-\|x\|_2^2-\|\Tilde{v}-x\|_2^2}{1-\|x\|_2^2+\|\Tilde{v}-x\|_2^2}\right)\Tilde{v}.
        \end{equation}
    \end{enumerate}
\end{proposition}

\begin{proof}
    \begin{enumerate}
        \item \textbf{Lorentz model.}
        
        First, a point on the geodesic $\gamma_v$ is of the form
        \begin{equation}
            y(t) = \cosh(t) x^0 + \sinh(t) v,
        \end{equation}
        with $t\in\mathbb{R}$.
        
        The projection along the horosphere amounts at following the level sets of the Busemann function $B^v$. And we have
        \begin{equation}
            \begin{aligned}
                B^v(x) = B^v(y(t)) &\iff \log(-\langle x,x^0+v\rangle_\mathbb{L}) = \log(-\langle \cosh(t)x^0 + \sinh(t) v, x^0+v\rangle_\mathbb{L}) \\
                &\iff \log (-\langle x,x^0+v\rangle_\mathbb{L}) = \log(-\cosh(t) \|x^0\|_\mathbb{L}^2 - \sinh(t)\|v\|_\mathbb{L}^2 ) \\
                &\iff \log(-\langle x,x^0+v\rangle_\mathbb{L} = \log (\cosh(t)-\sinh(t)) \\
                &\iff \langle x, x^0+v\rangle_\mathbb{L} = \sinh(t)-\cosh(t).
            \end{aligned}
        \end{equation}
        By noticing that $\cosh(t) = \frac{1+\tanh^2(\frac{t}{2})}{1-\tanh^2(\frac{t}{2})}$ and $\sinh(t) = \frac{2\tanh(\frac{t}{2})}{1-\tanh^2(\frac{t}{2})}$, let $u=\tanh(\frac{t}{2})$, then we have
        \begin{equation}
            \begin{aligned}
                B^v(x) = B^v(y(t)) &\iff \langle x,x^0+v\rangle_\mathbb{L} = \frac{2u}{1-u^2} - \frac{1+u^2}{1-u^2} = \frac{-(u-1)^2}{(1-u)(1+u)} = \frac{u-1}{u+1} \\
                &\iff u = \frac{1+\langle x,x^0+v\rangle_\mathbb{L}}{1-\langle x,x^0+v\rangle_\mathbb{L}}.
            \end{aligned}
        \end{equation}
        
        We can further continue the computation and obtain, by denoting $c=\langle x,x^0+v\rangle_\mathbb{L}$,
        \begin{equation} \label{eq:proj_horo_lorentz}
            \begin{aligned}
                \Tilde{B}^v(x) &= \frac{1+u^2}{1-u^2} x^0 + \frac{2u}{1-u^2} v \\
                &= \frac{1+\left(\frac{1+c}{1-c}\right)^2}{1-\left(\frac{1+c}{1-c}\right)^2} x^0 + 2 \frac{\left(\frac{1+c}{1-c}\right)}{1-\left(\frac{1+c}{1-c}\right)^2} v \\
                &= \frac{(1-c)^2 + (1+c)^2}{(1-c)^2 - (1+c)^2} x^0 + 2 \frac{(1+c)(1-c)}{(1-c)^2 - (1+c)^2} v \\
                &= - \frac{1+c^2}{2c} x^0 - \frac{1-c^2}{2c} v \\
                &= -\frac{1}{2\langle x, x^0 + v \rangle_\mathbb{L}} \big( (1+\langle x,x^0+v\rangle_\mathbb{L}^2)x^0 + (1-\langle x,x^0 +v\rangle_\mathbb{L}^2)v\big).
            \end{aligned}
        \end{equation}
        
        \item \textbf{Poincaré ball.}
        
        Let $p\in S^{d-1}$. First, we notice that points on the geodesic generated by $p$ and passing through 0 are of the form $x(\lambda)=\lambda p$ where $\lambda\in]-1,1[$.
        
        Moreover, there is a unique horosphere $S(p,x)$ passing through $x$ and starting from $p$. The points on this horosphere are of the form
        \begin{equation}
            \begin{aligned}
                y(\theta) &= \left(\frac{p+x(\lambda^*)}{2}\right) + \left\|\frac{p-x(\lambda^*)}{2}\right\|_2 \left(\cos(\theta) p + \sin(\theta) \frac{x-\langle x,p\rangle p}{\|x-\langle x,p\rangle p\|_2}\right) \\
                &= \frac{1+\lambda^*}{2}p + \frac{1-\lambda^2}{2} \left(\cos(\theta) p + \sin(\theta) \frac{x-\langle x,p\rangle p}{\|x-\langle x,p\rangle p\|_2}\right),
            \end{aligned}
        \end{equation}
        where $\lambda^*$ characterizes the intersection between the geodesic and the horosphere.
        
        Since the horosphere are the level sets of the Busemann function, we have $B^p(x)=B^p(\lambda^* p)$. Thus, we have
        \begin{equation}
            \begin{aligned}
                B^p(x) = B^p(\lambda^* p) &\iff \log\left(\frac{\|p-x\|_2^2}{1-\|x\|_2^2}\right) = \log\left(\frac{\|p-\lambda^* p\|_2^2}{1-\|\lambda^*p\|_2^2}\right) \\
                &\iff \frac{\|p-x\|_2^2}{1-\|x\|_2^2} = \frac{(1-\lambda^*)^2}{1-(\lambda^*)^2} \\
                &\iff \frac{\|p-x\|_2^2}{1-\|x\|_2^2} = \frac{1-\lambda^*}{1+\lambda^*} \\
                &\iff \lambda^* \left(\frac{\|p-x\|_2^2}{1-\|x\|_2^2} + 1\right) = 1 - \frac{\|p-x\|_2^2}{1-\|x\|_2^2} \\
                &\iff \lambda^* = \frac{1-\|x\|_2^2-\|p-x\|_2^2}{1-\|x\|_2^2+\|p-x\|_2^2}.
            \end{aligned}
        \end{equation}
    \end{enumerate}
\end{proof}

\subsection{Proof of Proposition \ref{prop:equality_hhsw}} \label{proof:equality_hhsw}

First, we show some Lemma.

\begin{lemma}[Commutation of projections.] \label{lemma:commute_projs}
    Let $v\in\mathrm{span}(x^0)^\bot \cap S^d$ of the form $v=(0,\Tilde{v})$ where $\Tilde{v}\in S^{d-1}$.
    Then, for all $x\in\mathbb{B}^d$, $y\in \mathbb{L}^d$
    \begin{align}
        P_{\mathbb{B}\to\mathbb{L}}\big(\Tilde{B}^{\Tilde{v}}(x)\big) = \Tilde{B}^v\big(P_{\mathbb{B}\to\mathbb{L}}(x)\big), \label{eq:busemann_b_to_l} \\
        \Tilde{B}^{\Tilde{v}}(P_{\mathbb{L}\to \mathbb{B}}(y)) = P_{\mathbb{L}\to\mathbb{B}}(\Tilde{B}^{v}(y)) \label{eq:busemann_l_to_b} \\ 
        P_{\mathbb{B}\to\mathbb{L}}\big(\Tilde{P}^{\Tilde{v}}(x)\big) = \Tilde{P}^v\big(P_{\mathbb{B}\to\mathbb{L}}(x)\big), \label{eq:geodesic_b_to_l} \\
        \Tilde{P}^{\Tilde{v}}(P_{\mathbb{L}\to \mathbb{B}}(y)) = P_{\mathbb{L}\to\mathbb{B}}(\Tilde{P}^{v}(y)). \label{eq:geodesic_l_to_b}
    \end{align}
\end{lemma}

\begin{proof}
    We first show \eqref{eq:busemann_b_to_l}. Let's recall the formula of the different projections.
    
    On one hand, 
    \begin{equation}
        \forall x\in\mathbb{B}^d,\ \Tilde{B}^{\Tilde{v}}(x) = \left(\frac{1-\|x\|_2^2-\|\Tilde{v}-x\|_2^2}{1-\|x\|_2^2+\|\Tilde{v}-x\|_2^2}\right)\Tilde{v},
    \end{equation}
    \begin{equation}
        \forall x\in \mathbb{L}^d,\ \Tilde{B}^v(x) = -\frac{1}{2\langle x, x^0 + v \rangle_\mathbb{L}} \big( (1+\langle x,x^0+v\rangle_\mathbb{L}^2)x^0 + (1-\langle x,x^0 +v\rangle_\mathbb{L}^2)v\big),
    \end{equation}
    and
    \begin{equation}
        \forall x\in\mathbb{B}^d,\ P_{\mathbb{B}\to\mathbb{L}}(x) = \frac{1}{1-\|x\|_2^2}(1+\|x\|_2^2, 2x_1,\dots, 2x_d).
    \end{equation}
    
    Let $x\in\mathbb{B}^d$.
    First, let's compute $P_{\mathbb{B}\to\mathbb{L}}\big(\Tilde{B}^{\Tilde{v}}(x)\big)$. We note that $\|\Tilde{v}\|_2^2=1$ and therefore
    \begin{equation}
        \|\Tilde{B}^{\Tilde{v}}(v)\|_2^2 = \left(\frac{1-\|x\|_2^2-\|\Tilde{v}-x\|_2^2}{1-\|x\|_2^2+\|\Tilde{v}-x\|_2^2}\right)^2.
    \end{equation}
    Then,
    \begin{equation}
        \begin{aligned}
            P_{\mathbb{B}\to\mathbb{L}}\big(\Tilde{B}^{\Tilde{v}}(x)\big) &= \frac{1}{1-\left(\frac{1-\|x\|_2^2-\|\Tilde{v}-x\|_2^2}{1-\|x\|_2^2+\|\Tilde{v}-x\|_2^2}\right)^2} \left(1+\left(\frac{1-\|x\|_2^2-\|\Tilde{v}-x\|_2^2}{1-\|x\|_2^2+\|\Tilde{v}-x\|_2^2}\right)^2, 2 \left(\frac{1-\|x\|_2^2-\|\Tilde{v}-x\|_2^2}{1-\|x\|_2^2+\|\Tilde{v}-x\|_2^2}\right) \Tilde{v}\right) \\
            &= \frac{1}{1-\left(\frac{1-\|x\|_2^2-\|\Tilde{v}-x\|_2^2}{1-\|x\|_2^2+\|\Tilde{v}-x\|_2^2}\right)^2} \left(\left(1+\left(\frac{1-\|x\|_2^2-\|\Tilde{v}-x\|_2^2}{1-\|x\|_2^2+\|\Tilde{v}-x\|_2^2}\right)^2\right)x^0 + 2 \left(\frac{1-\|x\|_2^2-\|\Tilde{v}-x\|_2^2}{1-\|x\|_2^2+\|\Tilde{v}-x\|_2^2}\right) v\right) \\
            &= \frac{\big(1-\|x\|_2^2+\|\Tilde{v}-x\|_2^2\big)^2}{4\|\Tilde{v}-x\|_2^2 (1-\|x\|_2^2)} \left(\frac{2(1-\|x\|_2^2)^2 + 2\|\Tilde{v}-x\|_2^4}{\big(1-\|x\|_2^2+\|\Tilde{v}-x\|_2^2\big)^2} x^0 + 2 \left(\frac{1-\|x\|_2^2-\|\Tilde{v}-x\|_2^2}{1-\|x\|_2^2+\|\Tilde{v}-x\|_2^2}\right) v\right) \\
            &= \frac{1}{2\|\Tilde{v}-x\|_2^2 (1-\|x\|_2^2)} \left( \big((1-\|x\|_2^2)^2 + \|\Tilde{v}-x\|_2^4\big)x^0 + (1-\|x\|_2^2-\|\Tilde{v}-x\|_2^2)(1-\|x\|_2^2+\|\Tilde{v}-x\|_2^2) v \right) \\
            &= \frac{1}{2\|\Tilde{v}-x\|_2^2 (1-\|x\|_2^2)} \left((1-\|x\|_2^2)^2 + \|\Tilde{v}-x\|_2^4\big)x^0 + \big((1-\|x\|_2^2)^2 - \|\Tilde{v}-x\|_2^4\big) v\right).
        \end{aligned}
    \end{equation}
    
    Now, let's compute $\Tilde{B}^v\big(P_{\mathbb{B}\to\mathbb{L}}(x)\big)$. First, let's remark that for all $y\in\mathbb{L}^d$, $\langle y,x^0 + v\rangle_\mathbb{L} = -y_0 + \langle y_{1:d}, \Tilde{v}\rangle$. Therefore, for all $x\in\mathbb{B}^d$,
    \begin{equation}
        \begin{aligned}
            \langle P_{\mathbb{B}\to\mathbb{L}}(x), x^0 + v\rangle_\mathbb{L} &= \langle \frac{1}{1-\|x\|_2^2}(1+\|x\|_2^2, 2x_1,\dots, 2x_d), x^0+v\rangle_\mathbb{L} \\
            &= \frac{1}{1-\|x\|_2^2} \left(-1-\|x\|_2^2 + 2 \langle x, \Tilde{v}\rangle\right) \\
            &= -\frac{1}{1-\|x\|_2^2} \|x-\Tilde{v}\|_2^2.
        \end{aligned}
    \end{equation}
    Moreover,
    \begin{equation}
        \begin{aligned}
            \langle P_{\mathbb{B}\to\mathbb{L}}(x), x^0 + v\rangle_\mathbb{L}^2 &= \frac{1}{(1-\|x\|_2^2)^2} \|\Tilde{v}-x\|_2^4.
        \end{aligned}
    \end{equation}
    
    Therefore, we have
    \begin{equation}
        \begin{aligned}
            \Tilde{B}^v\big(P_{\mathbb{B}\to\mathbb{L}}(x)\big) &= \Tilde{B}^v\left(\frac{1}{1-\|x\|_2^2}(1+\|x\|_2^2, 2x_1,\dots, 2x_d)\right) \\
            &= -\frac{1-\|x\|_2^2}{2 \left(-1-\|x\|_2^2 + 2 \langle x, \Tilde{v}\rangle\right)} \left( \big(1+\langle P_{\mathbb{B}\to\mathbb{L}}(x), x^0 + v\rangle_\mathbb{L}^2\big)x^0 + \big(1-\langle P_{\mathbb{B}\to\mathbb{L}}(x), x^0 + v\rangle_\mathbb{L}^2\big) v\right)\\
            &= \frac{1-\|x\|_2^2}{2 \|x-\Tilde{v}\|_2^2} \left( \frac{(1-\|x\|_2^2)^2 + \|\Tilde{v}-x\|^4}{(1-\|x\|_2^2)^2} x^0 + \frac{(1-\|x\|_2^2)^2 - \|\Tilde{v}-x\|^4}{(1-\|x\|_2^2)^2} v \right) \\
            &= \frac{1}{2\|x-\Tilde{v}\|_2^2 (1-\|x\|_2^2)} \left( (1-\|x\|_2^2)^2 + \|\Tilde{v}-x\|_2^4\big)x^0 + \big((1-\|x\|_2^2)^2 - \|\Tilde{v}-x\|_2^4\big) v\right) \\
            &= P_{\mathbb{B}\to\mathbb{L}}\big(\Tilde{B}^{\Tilde{v}}(x)\big).
        \end{aligned}
    \end{equation}

        For \eqref{eq:busemann_l_to_b}, we use that $P_{\mathbb{B}\to\mathbb{L}}$ and $P_{\mathbb{L}\to\mathbb{B}}$ are inverse from each other. Hence, for all $x\in \mathbb{B}^d$, there exists $y\in\mathbb{L}^d$ such that $x=P_{\mathbb{L}\to\mathbb{B}}(y) \iff y = P_{\mathbb{B}\to\mathbb{L}}(x)$, and we obtain the second equality by plugging it into \eqref{eq:busemann_b_to_l}.

    Now, let's show \eqref{eq:geodesic_b_to_l}. The proof relies on the observation that $\{\exp_{x^0}(tv),\ t\in\mathbb{R}\}=P_{\mathbb{B}\to\mathbb{L}}\left(\{\exp_0(t\Tilde{v}),\ t\in\mathbb{R}\}\right)$ (\emph{i.e.} the images by $P_{\mathbb{B}\to\mathbb{L}}$ of geodesics in the Poincaré ball are geodesics in the Lorentz model). Thus,
    \begin{equation}
        \begin{aligned}
            \Tilde{P}^v(P_{\mathbb{B}\to\mathbb{L}}(x)) &= \argmin_{z\in \{\exp_{x^0}(tv),\ t\in\mathbb{R}\}}\ d_{\mathbb{L}}(P_{\mathbb{B}\to\mathbb{L}}(x), z) \\
            &= P_{\mathbb{B}\to\mathbb{L}}\big(\argmin_{z\in\{\exp_0(t\Tilde{v}),\ t\in\mathbb{R}\}}\ d_\mathbb{B}(P_{\mathbb{L}\to\mathbb{L}}(x), P_{\mathbb{B}\to\mathbb{L}}(z))\big) \\
            &= P_{\mathbb{B}\to\mathbb{L}}\big(\argmin_{z\in\{\exp_0(t\Tilde{v}),\ t\in\mathbb{R}\}}\ d_\mathbb{B}(x,z)\big) \\
            &= P_{\mathbb{B}\to\mathbb{L}}\big(\Tilde{P}^v(x)\big).
        \end{aligned}
    \end{equation}
    Similarly, we obtain \eqref{eq:geodesic_l_to_b}.
\end{proof}



\begin{lemma} \label{lemma:eq_busemann_top}
    Let $v=(0,\Tilde{v})\in \mathrm{span}(x^0)^\top$. For all $x\in\mathbb{L}^d$, $y\in \mathbb{B}^d$,
    \begin{align}
        B^v(x) = -t^v\big(\Tilde{B}^v(x)\big), \\
        B^{\Tilde{v}}(y) = - t^{\Tilde{v}}(\Tilde{B}^{\Tilde{v}}(y)).
    \end{align}
\end{lemma}

\begin{proof}
    First, let us show that 
    \begin{equation}
        d_\mathbb{L}(\Tilde{B}^v(x), x^0) = |B^v(x)|.
    \end{equation}
    
    By recalling that $B^v\big(\Tilde{P}^v(x)\big) = B^v(x) = \log\big(-\langle x,x^0+v\rangle_\mathbb{L}\big)$ and that by \eqref{eq:proj_horo_lorentz},
    \begin{equation}
        \Tilde{B}^v(x) = -\frac{1}{2\langle x,x^0+v\rangle_\mathbb{L}} \left((1+\langle x,x^0+v\rangle_\mathbb{L}^2)x^0 - \langle x, x^0+v\rangle_\mathbb{L}^2 v\right).
    \end{equation}
    Now, by remarking that $\langle x,x^0+v\rangle_\mathbb{L}\le 0$, then we have, 
    \begin{equation}
        \begin{aligned}
            d_\mathbb{L}\big(\Tilde{B}^v(x),x^0\big) &= \arccosh(-\langle \Tilde{B}^v(x), x^0 \rangle_\mathbb{L}) \\
            &= \arccosh\left(\frac{1}{2\langle x,x^0+v\rangle_\mathbb{L}}(1+\langle x,x^0+v\rangle_\mathbb{L}^2)\langle x^0,x^0\rangle_\mathbb{L}\right) \\
            &= \arccosh\left(-\frac{1}{2\langle x,x^0+v\rangle_\mathbb{L}}(1+\langle x,x^0+v\rangle_\mathbb{L}^2)\right) \\
            &= \log\left(\frac{1+\langle x,x^0+v\rangle_\mathbb{L}^2}{-2\langle x,x^0+v\rangle_\mathbb{L}} + \sqrt{\frac{(1+\langle x,x^0+v\rangle_\mathbb{L})^2}{4\langle x,x^0+v\rangle_\mathbb{L}^2} -1}\right) \\
            &= \log\left(\frac{1+\langle x,x^0 + v\rangle_\mathbb{L}^2 + \sqrt{\big(\langle x,x^0 + v\rangle_\mathbb{L}^2-1\big)^2}}{-2\langle x, x^0+v\rangle_\mathbb{L}}\right) \\
        \end{aligned}
    \end{equation}
    If $\langle x,x^0+v\rangle_\mathbb{L}^2 \ge 1$, then
    \begin{equation}
        \begin{aligned}
            d_\mathbb{L}\big(\Tilde{B}^v(x), x^0\big) &= \log\left( \frac{1+\langle x,x^0+v\rangle_\mathbb{L}^2+\langle x,x^0+v\rangle_\mathbb{L}^2 -1}{-2 \langle x,x^0+v\rangle_\mathbb{L}}\right) \\
            &= \log(-\langle x,x^0+v\rangle_\mathbb{L}) = B^v(x).
        \end{aligned}
    \end{equation}
    And if $\langle x,x^0+v\rangle_\mathbb{L}^2 \le 1$, then
    \begin{equation}
        \begin{aligned}
            d_\mathbb{L}\big(\Tilde{B}^v(x), x^0\big) &= \log\left( \frac{1+\langle x,x^0+v\rangle_\mathbb{L}^2 +1-\langle x,x^0+v\rangle_\mathbb{L}^2}{-2 \langle x,x^0+v\rangle_\mathbb{L}}\right) \\
            &= \log\left(\frac{1}{-\langle x,x^0+v\rangle_\mathbb{L}}\right) = -B^v(x).
        \end{aligned}
    \end{equation}

    Hence, we showed that for all $x\in\mathbb{L}^d$,
    \begin{equation}
        d_\mathbb{L}\big(\Tilde{B}^v(x),x^0\big) = |B^v(x)|.
    \end{equation}
    Then, using \eqref{eq:proj_horo_lorentz}, we have
    \begin{equation}
        \langle \Tilde{B}^v(x), v\rangle = \frac{1-\langle x,x^0+v\rangle_\mathbb{L}^2}{-2\langle x,x^0+v\rangle_\mathbb{L}}.
    \end{equation}
    
    Then, on one hand, we can show that $\langle x,x^0+v\rangle_\mathbb{L}\le 0$ since $x=\lambda_0 x^0 + \lambda_1 v + x_{\mathrm{span}(x^0,v)^\bot}$. Thus, 
    \begin{equation}
        \langle x, x^0 + v\rangle_\mathbb{L} = -\lambda_0 + \lambda_1.
    \end{equation}
    But, $\lambda_0 = \sqrt{1+\sum_{i=1^d} \lambda_i^2}\ge \sqrt{\lambda_1^2} \ge \lambda_1$. Therefore, $\lambda_1-\lambda_0 = \langle x, x^0+v\rangle_\mathbb{L}\le 0$.
    
    Therefore, we have $-\langle x,x^0+v\rangle_\mathbb{L}\ge 0$. And we have
    \begin{equation}
        \langle \Tilde{B}^v(x),v\rangle \ge 0 \iff 1-\langle x,x^0+v\rangle_\mathbb{L}^2 \ge 0 \iff 1\ge \langle x,x^0+v\rangle_\mathbb{L}^2 \iff B^v(x) \le 0,
    \end{equation}
    using that $B^v(x) = \log(-\langle x,x^0+v\rangle_\mathbb{L})$.
    
    Similarly, 
    \begin{equation}
        \langle \Tilde{B}^v(x),v\rangle \le 0 \iff 1-\langle x,x^0+v\rangle_\mathbb{L}^2 \le 0 \iff 1\le \langle x,x^0+v\rangle_\mathbb{L}^2 \iff B^v(x) \ge 0.
    \end{equation}
    
    Hence, 
    \begin{equation}
        \mathrm{sign}(\langle\Tilde{B}^v(x),v\rangle) = - \mathrm{sign}(B^v(x)).
    \end{equation}
    
    Finally, we deduce that 
    \begin{equation} \label{eq:eq_busemann_coords_lorentz}
        t^v\big(\Tilde{B}^v(x)\big) = \mathrm{sign}(\langle \Tilde{B}^v(x),v\rangle) d_\mathbb{L}\big(\Tilde{B}^v(x),x^0\big) = -B^v(x).
    \end{equation}

    For the second equality, let $y\in\mathbb{B}^d$, then,
    \begin{equation}
        \begin{aligned}
            t^{\Tilde{v}}(\Tilde{B}^{\Tilde{v}}(y)) &= \mathrm{sign}(\langle \Tilde{B}^{\Tilde{v}}(y), \Tilde{v}\rangle)\ d_{\mathbb{B}}(\Tilde{B}^{\Tilde{v}}(y), 0) \\
            &= \mathrm{sign}(\langle P_{\mathbb{B}\to\mathbb{L}}(\Tilde{B}^{\Tilde{v}}(y)), v\rangle) d_{\mathbb{L}}(P_{\mathbb{B}\to\mathbb{L}}(\Tilde{B}^{\Tilde{v}}(y)), x^0) \\
            &= \mathrm{sign}(\langle \Tilde{B}^{v}(P_{\mathbb{B}\to\mathbb{L}}(y)), v\rangle) d_{\mathbb{L}}(\Tilde{B}^v(P_{\mathbb{B}\to\mathbb{L}}(y)), x^0) \quad \text{using Lemma \ref{lemma:commute_projs}} \\
            &= t^v(\Tilde{B}^v(P_{\mathbb{B}\to\mathbb{L}}(y))) \quad \text{by definition of $t^v$}\\
            &= -B^v(P_{\mathbb{B}\to\mathbb{L}}(y)) \quad \text{by \eqref{eq:eq_busemann_coords_lorentz}} \\
            &= -B^{\Tilde{v}}(y),
        \end{aligned}
    \end{equation}
    where the last line comes from
    \begin{equation}
        \begin{aligned}
            B^v(P_{\mathbb{B}\to\mathbb{L}}(y)) &= \lim_{t\to \infty}\ \big(d_{\mathbb{L}}(\exp_{x^0}(tv), P_{\mathbb{B}\to\mathbb{L}}(y)) -t\big) \\
            &= \lim_{t\to\infty}\ \big( d_{\mathbb{L}}(P_{\mathbb{B}\to\mathbb{L}}(P_{\mathbb{L}\to\mathbb{B}}(\exp_{x^0}(tv))), P_{\mathbb{B}\to\mathbb{L}}(y)) - t\big) \\
            &= \lim_{t\to\infty}\ \big(d_{\mathbb{B}}(P_{\mathbb{L}\to\mathbb{B}}(\exp_{x^0}(tv)), y) - t\big) \\
            &= \lim_{t\to\infty}\ \big(d_{\mathbb{B}}(\exp_0(t\Tilde{v}), y) - t\big) \\
            &= B^{\Tilde{v}}(y).
        \end{aligned}
    \end{equation}
\end{proof}

\begin{proof}[Proof of Proposition \ref{prop:equality_hhsw}]
    Let $\mu,\nu\in\mathcal{P}(\mathbb{B}^d)$, $\Tilde{\mu} = (P_{\mathbb{B}\to\mathbb{L}})_\#\mu$, $\Tilde{\nu} = (P_{\mathbb{B}\to\mathbb{L}})_\#\nu$, $\Tilde{v}\in S^{d-1}$ an ideal point and $v=(0,\Tilde{v})\in\mathrm{span}(x^0)^\bot$.

    First, by Lemma \ref{lemma:eq_busemann_top}, $B^v = -t^v\circ \Tilde{B}^v$. Using the proof \ref{proof:prop_wasserstein_geodesics}, $t^v$ is an isometry and we have that (by using the invariant of the Wasserstein distance),
    \begin{equation}
        W_p^p(B^v_\#\Tilde{\mu}, B^v_\#\Tilde{\nu}) = W_p^p(\Tilde{B}^v_\#\Tilde{\mu}, \Tilde{B}^v_\#\Tilde{\nu}).
    \end{equation}
    Then,
    \begin{equation}
        \begin{aligned}
        W_p^p(B^v_\#\Tilde{\mu}, B^v_\#\Tilde{\nu}) &= W_p^p(\Tilde{B}^v_\#\Tilde{\mu}, \Tilde{B}^v_\#\Tilde{\nu}) \\
        &= W_p^p(\Tilde{B}^v_\#(P_{\mathbb{B}\to\mathbb{L}})_\#\mu, \Tilde{B}^v_\#(P_{\mathbb{B}\to\mathbb{L}})_\#\nu) \\
        &= \inf_{\gamma\in\Pi(\mu,\nu)} \ \int_{\mathbb{B}^d\times \mathbb{B}^d} d_\mathbb{L}\big(\Tilde{B}^v(P_{\mathbb{B}\to\mathbb{L}}(x)), \Tilde{B}^v(P_{\mathbb{B}\to\mathbb{L}}(y))\big)^p\ \mathrm{d}\gamma(x,y) \quad \text{by \citep[Lemma 6]{paty2019subspace}} \\
        &= \inf_{\gamma\in\Pi(\mu,\nu)} \ \int_{\mathbb{B}^d\times \mathbb{B}^d} d_\mathbb{L}\big(P_{\mathbb{B}\to\mathbb{L}}(\Tilde{B}^{\Tilde{v}}(x)), P_{\mathbb{B}\to\mathbb{L}}(\Tilde{B}^{\Tilde{v}}(y))\big)^p\ \mathrm{d}\gamma(x,y) \quad \text{by Lemma \ref{lemma:commute_projs}} \\
        &= \inf_{\gamma\in\Pi(\mu,\nu)} \int_{\mathbb{B}^d \times \mathbb{B}^d} d_\mathbb{B}\big(\Tilde{B}^{\Tilde{v}}(x), \Tilde{B}^{\Tilde{v}}(y)\big)^p\ \mathrm{d}\gamma(x,y) \quad \text{as $P_{\mathbb{B}\to\mathbb{L}}$ is an isometry} \\
        &= W_p^p(\Tilde{B}^{\Tilde{v}}_\#\mu, \Tilde{B}^{\Tilde{v}}_\#\nu) \quad \text{by \citep[Lemma 6]{paty2019subspace}}  \\
        &= W_p^p(B^{\Tilde{v}}_\#\mu, B^{\Tilde{v}}_\#\nu),
        \end{aligned}
    \end{equation}
    where for the last line, we use that by Lemma \ref{lemma:commute_projs}, $B^{\Tilde{v}} = -t^{\Tilde{v}}\circ \Tilde{B}^{\Tilde{v}}$ and that $t^{\Tilde{v}}$ is an isometry. Indeed,
    \begin{equation}
        \begin{aligned}
            \forall x,y \in \{\exp_0(t\Tilde{v}),\ t\in\mathbb{R},\ |t^{\Tilde{v}}(x)-t^{\Tilde{v}}(y)| &= |\mathrm{sign}(\langle x, \Tilde{v}\rangle) d_{\mathbb{B}}(x,0) - \mathrm{sign}(\langle y, \Tilde{v}\rangle) d_{\mathbb{B}}(y,0)| \\
            &= |\mathrm{sign}(\langle P_{\mathbb{B}\to\mathbb{L}}(x), v\rangle) d_{\mathbb{L}}(P_{\mathbb{B}\to\mathbb{L}}(x), x^0) - \mathrm{sign}(\langle P_{\mathbb{B}\to\mathbb{L}}(y), v\rangle) d_{\mathbb{L}}(P_{\mathbb{B}\to\mathbb{L}}(y), x^0)| \\
            &= |t^v(P_{\mathbb{B}\to\mathbb{L}}(x)) - t^v(P_{\mathbb{B}\to\mathbb{L}}(y))| \\
            &= d_{\mathbb{L}}(P_{\mathbb{B}\to\mathbb{L}}(x), P_{\mathbb{B}\to\mathbb{L}}(y)) \quad \text{by Proposition \ref{prop:wasserstein_geodesics}} \\
            &= d_{\mathbb{B}}(x,y)  \quad \text{as $P_{\mathbb{B}\to\mathbb{L}}$ is an isometry} .
        \end{aligned}
    \end{equation}

    It is true for all $\Tilde{v}\in S^{d-1}$, and hence for $\lambda$-almost all $\Tilde{v}\in S^{d-1}$. Therefore, we have
    \begin{equation}
        HHSW_p^p(\mu,\nu)=HHSW_p^p(\Tilde{\mu},\Tilde{\nu}).
    \end{equation}

    Similarly, with the same reasonment, using that $t^v$ and $t^{\Tilde{v}}$ are isometries and Lemma \ref{lemma:commute_projs}, we obtain 
    \begin{equation}
        GHSW_p^p(\mu,\nu) = GHSW_p^p(\Tilde{\mu}, \Tilde{\nu}).
    \end{equation}
\end{proof}

\section{Properties} \label{appendix:properties}

We derive in this section additional properties of HHSW and GHSW. First, we will start by showing that for $\mu,\nu\in\mathcal{P}_p(\mathbb{L}^d)$, we have well $GHSW_p(\mu,\nu)<\infty$ and $HHSW_p(\mu,\nu)<\infty$. We also show that the Busemann coordinates can directly be used in HHSW to compute the coordinates on $\mathbb{R}$. Then, we continue by showing that GHSW and HHSW are pseudo-distances. And finally, we make connections with Radon transforms known in the literature.

\subsection{Finiteness of GHSW and HHSW} \label{appendix:finiteness}

\begin{proposition}
    Let $p\ge 1$, then for $\mu,\nu\in\mathcal{P}_p(\mathbb{L}^d)$, $GHSW_p(\mu,\nu)<\infty$ and $HHSW_p(\mu,\nu)<\infty$.
\end{proposition}

\begin{proof}
    First, we will deal with GHSW and then with HHSW. Let $p\ge 1$ and $\mu,\nu\in\mathcal{P}_p(\mathbb{L}^d) = \{\mu\in\mathcal{P}(\mathbb{L}^d), \ \int_{\mathbb{L}^d} d_\mathbb{L}(x,x_0)^p \mathrm{d}\mu(x) <\infty \ \text{ for some }x_0\in \mathbb{L}^d\}$. Note that the choice of $x_0$ is arbitrary, since for any $x,y\in\mathbb{L}^d$, we have by the triangular inequality
    \begin{equation}
        d_\mathbb{L}(x,y) \le d_\mathbb{L}(x,x_0) + d_\mathbb{L}(x_0, y).
    \end{equation}
    Then, both proofs will follow from \citep[Definition 6.4]{villani2009optimal} using that
    \begin{equation} \label{eq:inequality_dist}
        \forall x,y\in\mathbb{L}^d,\ d_\mathbb{L}(x,y)^p \le 2^{p-1}\big(d_\mathbb{L}(x,x_0)^p+d_{\mathbb{L}}(x_0,y)^p\big).
    \end{equation}
    
    \paragraph{GHSW.}  Let $\mu,\nu\in\mathcal{P}_p(\mathbb{L}^d)$. Then, using \eqref{eq:inequality_dist}, we have, denoting $\gamma\in\Pi(\mu,\nu)$ an arbitrary coupling and using \citep[Lemma 6]{paty2019subspace},
    \begin{equation}
        \begin{aligned}
            W_p^p(P^v_\#\mu, P^v_\#\nu) = W_p^p(\Tilde{P}^v_\#\mu, \Tilde{P}^v_\#\nu) &= \inf_{\pi\in\Pi(\Tilde{P}^v_\#\mu, \Tilde{P}^v_\#\nu)}\ \int_{\mathbb{L}^d \times \mathbb{L}^d} d_\mathbb{L}(x,y)^p\ \mathrm{d}\pi(x,y) \\
            &= \inf_{\pi\in\Pi(\mu,\nu)}\ \int_{\mathbb{L}^d\times \mathbb{L}^d} d_\mathbb{L}\big(\Tilde{P}^v(x), \Tilde{P}^v(y)\big)^p\ \mathrm{d}\pi(x,y) \\
            &\le \int_{\mathbb{L}^d\times\mathbb{L}^d} d_\mathbb{L}\big(\Tilde{P}^v(x), \Tilde{P}^v(y)\big)^p\ \mathrm{d}\gamma(x,y) \\
            &\le 2^{p-1} \int_{\mathbb{L}^d\times\mathbb{L}^d}  \left( d_\mathbb{L}\big(\Tilde{P}^v(x), x_0\big)^p + d_\mathbb{L}\big(\Tilde{P}^v(y),x_0\big)^p\right)\ \mathrm{d}\gamma(x,y) \\
            &= 2^{p-1} \Big(\int_{\mathbb{L}^d} d_\mathbb{L}\big(\Tilde{P}^v(x),x_0\big)^p\mathrm{d}\mu(x) + \int_{\mathbb{L}^d} d_\mathbb{L}\big(\Tilde{P}^v(y),x_0\big)^p \mathrm{d}\nu(y)\Big).
        \end{aligned}
    \end{equation}
    If we take $x_0$ belonging to the geodesic, then necessarily, $d_\mathbb{L}\big(\Tilde{P}^v(x),x\big) \le d_\mathbb{L}(x_0,x)$ using that $P^v(x)=\argmin_{y\in\mathrm{span}(x^0,v)\cap\mathbb{L}^d}\ d_\mathbb{L}(x,y)$. Hence, by using again \eqref{eq:inequality_dist}, we have
    \begin{equation}
        \begin{aligned}
            W_p^p(P^v_\#\mu, P^v_\#\nu) &\le 2^{2p-2} \Big( \int_{\mathbb{L}^d} d_\mathbb{L}(\Tilde{P}^v(x), x)^p \mathrm{d}\mu(x) + \int_{\mathbb{L}^d} d_\mathbb{L}(x, x_0)^p\mathrm{d}\mu(x) \\
            &+ \int_{\mathbb{L}^d} d_\mathbb{L}(\Tilde{P}^v(y),y)^p \mathrm{d}\nu(y) + \int_{\mathbb{L}^d} d_\mathbb{L}(y,x_0)^p\mathrm{d}\nu(y)\Big) \\
            &\le \Big( \int_{\mathbb{L}^d} d_\mathbb{L}(x_0, x)^p \mathrm{d}\mu(x) + \int_{\mathbb{L}^d} d_\mathbb{L}(x, x_0)^p\mathrm{d}\mu(y) \\
            &+ \int_{\mathbb{L}^d} d_\mathbb{L}(x_0,y)^p \mathrm{d}\nu(y) + \int_{\mathbb{L}^d} d_\mathbb{L}(y,x_0)^p\mathrm{d}\nu(y)\Big) \\
            &< \infty.
        \end{aligned}
    \end{equation}
    And hence $GHSW_p(\mu,\nu)<\infty$.
    
    \paragraph{HHSW.} Let's take first $x_0=x^0$ as the base point. Then, by using again \eqref{eq:inequality_dist}, we have:
    \begin{equation}
        W_p^p(B^v_\#\mu, B^v_\#\nu) \le 2^{p-1} \Big(\int_{\mathbb{L}^d} d_\mathbb{L}\big(\Tilde{B}^v(x),x^0\big)^p\mathrm{d}\mu(x) + \int_{\mathbb{L}^d} d_\mathbb{L}\big(\Tilde{B}^v(y),x^0\big)^p\mathrm{d}\nu(y)\Big).
    \end{equation}
    
    Now, by recalling that $B^v\big(\Tilde{B}^v(x)\big) = B^v(x) = \log\big(-\langle x,x^0+v\rangle_\mathbb{L}\big)$ and 
    \begin{equation}
        \Tilde{B}^v(x) = -\frac{1}{2\langle x,x^0+v\rangle_\mathbb{L}} \left((1+\langle x,x^0+v\rangle_\mathbb{L}^2)x^0 - \langle x, x^0+v\rangle_\mathbb{L}^2 v\right).
    \end{equation}
    Now, by remarking that $\langle x,x^0+v\rangle_\mathbb{L}\le 0$, then we have, 
    \begin{equation}
        \begin{aligned}
            d_\mathbb{L}\big(\Tilde{B}^v(x),x^0\big) &= \arccosh(-\langle \Tilde{B}^v(x), x^0 \rangle_\mathbb{L}) \\
            &= \arccosh\left(\frac{1}{2\langle x,x^0+v\rangle_\mathbb{L}}(1+\langle x,x^0+v\rangle_\mathbb{L}^2)\langle x^0,x^0\rangle_\mathbb{L}\right) \\
            &= \arccosh\left(-\frac{1}{2\langle x,x^0+v\rangle_\mathbb{L}}(1+\langle x,x^0+v\rangle_\mathbb{L}^2)\right) \\
            &= \log\left(\frac{1+\langle x,x^0+v\rangle_\mathbb{L}^2}{-2\langle x,x^0+v\rangle_\mathbb{L}} + \sqrt{\frac{(1+\langle x,x^0+v\rangle_\mathbb{L})^2}{4\langle x,x^0+v\rangle_\mathbb{L}^2} -1}\right) \\
            &= \log\left(\frac{1+\langle x,x^0 + v\rangle_\mathbb{L}^2 + \sqrt{\big(\langle x,x^0 + v\rangle_\mathbb{L}^2-1\big)^2}}{-2\langle x, x^0+v\rangle_\mathbb{L}}\right) \\
        \end{aligned}
    \end{equation}
    If $\langle x,x^0+v\rangle_\mathbb{L}^2 \ge 1$, then
    \begin{equation}
        \begin{aligned}
            d_\mathbb{L}\big(\Tilde{B}^v(x), x^0\big) &= \log\left( \frac{1+\langle x,x^0+v\rangle_\mathbb{L}^2+\langle x,x^0+v\rangle_\mathbb{L}^2 -1}{-2 \langle x,x^0+v\rangle_\mathbb{L}}\right) \\
            &= \log(-\langle x,x^0+v\rangle_\mathbb{L}) = B^v(x).
        \end{aligned}
    \end{equation}
    And if $\langle x,x^0+v\rangle_\mathbb{L}^2 \le 1$, then
    \begin{equation}
        \begin{aligned}
            d_\mathbb{L}\big(\Tilde{B}^v(x), x^0\big) &= \log\left( \frac{1+\langle x,x^0+v\rangle_\mathbb{L}^2 +1-\langle x,x^0+v\rangle_\mathbb{L}^2}{-2 \langle x,x^0+v\rangle_\mathbb{L}}\right) \\
            &= \log\left(\frac{1}{-\langle x,x^0+v\rangle_\mathbb{L}}\right) = -B^v(x).
        \end{aligned}
    \end{equation}
    
    Then, using that $B^v$ is 1-lipschitz, we have
    \begin{equation}
        |B^v(x)-B^v(x^0)|\le d_\mathbb{L}(x,x^0),
    \end{equation}
    and therefore
    \begin{equation}
        |B^v(x)| \le d_\mathbb{L}(x,x^0),
    \end{equation}
    since $B^v(x^0) = \log(-\langle x^0,x^0+v\rangle_\mathbb{L}) = 0$.
    Hence we have well $HHSW_p(\mu,\nu)<\infty$.
\end{proof}

\subsection{Pseudo-distance} \label{appendix:pseudo_distance}

\begin{proposition}
    Let $p\ge 1$, then $GHSW_p$ and $HHSW_p$ are pseudo-distances.
\end{proposition}

\begin{proof}
    Let $p\ge 1$, then for all $\mu,\nu\in\mathcal{P}_p(\mathbb{L}^d)$, it is straightforward to see that $GHSW_p(\mu,\nu)\ge 0$, $HHSW_p(\mu,\nu)_p\ge 0$, $GHSW_p(\mu,\nu)=GHSW_p(\nu,\mu)$ and $HHSW_p(\mu,\nu)=HHSW_p(\nu,\mu)$. It is also easy to see that $\mu=\nu\implies GHSW_p(\mu,\nu)=0$ and $HHSW_p(\mu,\nu)=0$ using that $W_p$ is a distance.
    
    Now, we can also derive the triangular inequality using the triangular inequality for $W_p$ and the Minkowski inequality:
    \begin{equation}
        \begin{aligned}
            \forall \mu,\nu,\alpha\in\mathcal{P}(\mathbb{L}^d),\ GHSW_p(\mu,\nu) &= \Big(\int_{T_{x^0}\mathbb{L}^d\cap S^d} W_p^p(P^v_\#\mu,P^v_\#\nu)\ \mathrm{d}\lambda(v)\Big)^{\frac{1}{p}} \\
            &\le \Big(\int_{T_{x^0}\mathbb{L}^d\cap S^d} \big(W_p(P^v_\#\mu,P^v_\#\alpha) + W_p(P^v_\#\alpha,P^v_\#\nu)\big)^p\ \mathrm{d}\lambda(v)\Big)^{\frac{1}{p}} \\
            &\le \Big(\int_{T_{x^0}\mathbb{L}^d\cap S^d} W_p^p(P^v_\#\mu,P^v_\#\alpha)\ \mathrm{d}\lambda(v)\Big)^{\frac{1}{p}} \\ &+ \Big(\int_{T_{x^0}\mathbb{L}^d\cap S^d} W_p^p(P^v_\#\alpha, P^v_\#\nu)\ \mathrm{d}\lambda(v)\Big)^{\frac{1}{p}} \\
            &= GHSW_p(\mu,\alpha)+GHSW_p(\alpha,\nu).
        \end{aligned}
    \end{equation}
    The same holds for HHSW.
    
    Therefore, $GHSW_p$ and $HHSW_p$ are pseudo-distances.
\end{proof}

To show that there are distances, we need additionally the positivity property, \emph{i.e.} we need to show that $GHSW_p(\mu,\nu) = 0 \implies \mu=\nu$. As $W_p$ is a distance, we have that $GHSW_p(\mu,\nu)=0 \implies P^v_\#\mu=P^v_\#\nu$ for $\lambda$-ae $v$. But showing that this implies that $\mu=\nu$ is  not straightforward. Following derivations obtained with SW, we can draw connections with known Radon transforms.

\paragraph{Radon transform for GHSW.} Let $f\in L^1(\mathbb{L}^d)$. Then, let's define $\Bar{R}:L^1(\mathbb{L}^d)\to L^1(\mathbb{R}\times S^{d-1})$ such that for all $t\in\mathbb{R}$ and $v\in S^{d-1}$,
\begin{equation}
    \Bar{R}f(t, v) = \int_{\mathbb{L}^d} f(x)\mathbb{1}_{\{P^v(x) = t\}}\ \mathrm{d}x.
\end{equation}
Let's define a dual function $\Bar{R}^*:C_0(\mathbb{R}\times S^{d-1})\to C_0(\mathbb{L}^d)$ as
\begin{equation}
    \Bar{R}^*g(x) = \int_{S^{d-1}} g\big(P^v(x), v\big)\ \mathrm{d}\lambda(v),
\end{equation}
where $g\in C_0(\mathbb{R}\times S^{d-1})$. Then, we can check that it well the dual.

\begin{proposition}
    For all $f\in L^1(\mathbb{L}^d)$, $g\in C_0(\mathbb{R}\times S^{d-1})$, 
    \begin{equation}
        \langle \Bar{R}f, g\rangle_{\mathbb{R}\times S^{d-1}} = \langle f, \Bar{R}^*g\rangle_{\mathbb{L}^d}.
    \end{equation}
\end{proposition}

\begin{proof}
    Let $f\in L^1(\mathbb{L}^d)$, $g\in C_0(\mathbb{R}\times S^{d-1})$, then,
    \begin{equation}
        \begin{aligned}
            \langle \Bar{R}f, g\rangle_{\mathbb{R}\times S^{d-1}} &= \int_{\mathbb{R}\times S^{d-1}} \Bar{R}f(t, v) g(t, v) \mathrm{d}t \mathrm{d}\lambda(v) \\
            &= \int_\mathbb{R} \int_{\mathbb{S}^{d-1}} \int_{\mathbb{L}^d} f(x) \mathbb{1}_{\{P^v(x)=t\}} g(t, v) \ \mathrm{d}x \mathrm{d}t \mathrm{d}\lambda(v) \\
            &= \int_{\mathbb{L}^d} f(x) \int_{S^{d-1}} \int_\mathbb{R} g(t,v) \mathbb{1}_{\{P^v(x)=t\}} \ \mathrm{d}t \mathrm{d}\lambda(v) \mathrm{d}x \\
            &= \int_{\mathbb{L}^d} f(x) \int_{S^{d-1}} g\big(P^v(x), v\big)\ \mathrm{d}\lambda(v) \mathrm{d}x \\
            &= \langle f, \Bar{R}g\rangle_{\mathbb{L}^d}.
        \end{aligned}
    \end{equation}
\end{proof}

Then, we can as in \citep{boman2009support}, define the corresponding Radon transform of a measure $\mu\in\mathcal{M}(\mathbb{L}^d)$ as the measure $\Bar{R}\mu\in\mathcal{M}(\mathbb{R}\times S^{d-1})$, such that for all $g\in C_0(\mathbb{R}\times S^{d-1})$, $\langle \Bar{R}\mu, g\rangle_{\mathbb{R}\times S^{d-1}} = \langle \mu, \Bar{R}^*g\rangle_{\mathbb{L}^d}$.

Next, denoting for $v\in S^{d-1}$, $(\Bar{R}\mu)^v$ the disintegrated measure \emph{w.r.t.} $\lambda$, \emph{i.e.} the measure satisfying for all $\phi\in C(\mathbb{R}\times S^{d-1})$,
\begin{equation}
    \int_{\mathbb{R}\times S^{d-1}} \phi(t, v) \mathrm{d}(\Bar{R}\mu)(t,v) = \int_{S^{d-1}} \int_{\mathbb{R}} \phi(t, v) (R\mu)^v(\mathrm{d}t) \ \mathrm{d}\lambda(v),
\end{equation}
we can show that $(\Bar{R}\mu)^v=P^v_\#\mu$. 

\begin{proposition}
    Let $\mu\in\mathcal{M}(\mathbb{L}^d)$, then for $\lambda$-almost every $v\in S^{d-1},$ $(\Bar{R}\mu)^v = P^v_\#\mu$.
\end{proposition}

\begin{proof}
    In the following, we will use that $S^{d-1} \cong T_{x^0}\mathbb{L}^d\cap S^d$. And therefore, $P^v$ is well defined.
    
    Let $g\in C_0(\mathbb{R}\times S^{d-1})$, then
    \begin{equation}
        \begin{aligned}
            \int_{S^{d-1}} \int_{\mathbb{R}} g(t, v) (\Bar{R}\mu)^v(\mathrm{d}t) \ \mathrm{d}\lambda(v) &= \int_{\mathbb{R}\times S^{d-1}} g(t,v)\ \mathrm{d}(\Bar{R}\mu)(t,v) = \langle \Bar{R}\mu, g\rangle_{\mathbb{R}\times S^{d-1}} \\
            &= \int_{\mathbb{L}^d} \Bar{R}^*g(x) \ \mathrm{d}\mu(x) \\
            &= \int_{\mathbb{L}^d} \int_{S^{d-1}} g\big(P^v(x), v\big) \ \mathrm{d}\lambda(x)\mathrm{d}\mu(x) \\
            &= \int_{S^{d-1}} \int_{\mathbb{L}^d} g\big(P^v(x), v\big) \mathrm{d}\mu(x)\mathrm{d}\lambda(v) \\
            &= \int_{S^{d-1}} \int_{\mathbb{R}} g(t, v) \ \mathrm{d}(P^v_\#\mu)(x) \mathrm{d}\lambda(v),
        \end{aligned}
    \end{equation}
    where we use the duality properties and Fubini.
\end{proof}

From the previous proposition, we deduce that
\begin{equation}
    \forall \mu,\nu \in \mathcal{P}(\mathbb{L}^d),\ GHSW_p^p(\mu,\nu) = \int_{S^{d-1}} W_p^p\big((\Bar{R}\mu)^v, (\Bar{R}\nu)^v\big)\ \mathrm{d}\lambda(v).
\end{equation}
And $GHSW_p(\mu,\nu) = 0 \implies (\Bar{R}\mu)^v=(\Bar{R}\nu)^v$ for $\lambda$-almost every $\lambda$.

The transformation $\Bar{R}$ is not really clear written like that. In the next proposition, we identify the integration set, which will give a connection to a known Radon transform.

\begin{proposition}[Set of integration]
    The integration set of $\Bar{R}$ is, for $t\in\mathbb{R}$, $v\in S^{d-1}$,
    \begin{equation}
        \{x\in\mathbb{L}^d,\ P^v(x)=t\} = \mathrm{span}(v_z)^\bot \cap \mathbb{L}^d,
    \end{equation}
    where $v_z = R_z v$ with $R_z$ a rotation matrix 
    in the plan $\mathrm{span}(v,x^0)$ such that $\langle v_z, z\rangle = 0$.
\end{proposition}

\begin{proof}
    We will prove this proposition directly by working on the geodesics. As $t^v$ is an isometry (Proposition \ref{prop:wasserstein_geodesics}), for all $t\in\mathbb{R}$, there exists a unique $z$ on the geodesic $\mathrm{span}(x^0,v)\cap \mathbb{L}^d$ such that $t=t^v(z)$, and we can rewrite the set of integration as
    \begin{equation}
        \{x\in\mathbb{L}^d,\ P^v(x) = t\} = \{x\in \mathbb{L}^d,\ \Tilde{P}^v(x)=z\}.
    \end{equation}

    For the first inclusion, let $x\in\{x\in\mathbb{L}^d,\ \Tilde{P}^v(x)=z\}$. By Proposition \ref{prop:hsw_geodesic_proj} and hypothesis, we have that
    \begin{equation} \label{eq:proj}
        \Tilde{P}^v(x) = \frac{1}{\sqrt{\langle x,x^0\rangle_\mathbb{L}^2-\langle x,v\rangle_\mathbb{L}^2}} \big(-\langle x,x^0\rangle_\mathbb{L} x^0 + \langle x,v\rangle_\mathbb{L} v\big) = z.
    \end{equation}
    Let's denote $E=\mathrm{span}(v,x^0)$ the plan generating the geodesic. Then, by denoting $P^E$ the orthogonal projection on $E$, we have
    \begin{equation}
        \begin{aligned}
            P^E(x) &= \langle x, v\rangle v + \langle x,x^0\rangle x^0 \\
            &= \langle x,v\rangle_\mathbb{L} v - \langle x,x^0\rangle_\mathbb{L} x^0 \\
            &= \left(\sqrt{\langle x,x^0\rangle_\mathbb{L}^2-\langle x,v\rangle_\mathbb{L}^2}\right) z,
        \end{aligned}
    \end{equation}
    using that $v_0=0$ since $\langle x^0, v\rangle = v_0 = 0$, and hence $\langle x,v\rangle_\mathbb{L}=\langle x,v\rangle$, that $\langle x,x^0\rangle = x_0 = -\langle x,x^0\rangle_\mathbb{L}$ and \eqref{eq:proj}.
    Then, since $v_z\in\mathrm{span}(v, x^0)$ and $\langle z,v_z\rangle = 0$ (by construction of $R_z$), we have
    \begin{equation}
        \begin{aligned}
            \langle x,v_z\rangle &= \langle P^E(x), v_z\rangle  \\
            &= \langle \left(\sqrt{\langle x,x^0\rangle_\mathbb{L}^2-\langle x,v\rangle_\mathbb{L}^2}\right) z, v_z\rangle = 0.
        \end{aligned}
    \end{equation}
    Thus, $x\in \mathrm{span}(v_z)^\bot\cap \mathbb{L}^d$.
    
    For the second inclusion, let $x\in\mathrm{span}(v_z)^\bot \cap \mathbb{L}^d$. Since $z\in\mathrm{span}(v_z)^\bot$ (by construction of $R_z$), we can decompose $\mathrm{span}(v_z)^\bot$ as $\mathrm{span}(v_z)^\bot = \mathrm{span}(z)\oplus (\mathrm{span}(z)^\bot \setminus \mathrm{span}(v_z))$. Hence, there exists $\lambda\in\mathbb{R}$ such that $x=\lambda z + x^\bot$. Moreover, as $z\in\mathrm{span}(x^0,v)$, we have $\langle x,x^0\rangle_\mathbb{L} = \lambda \langle z, x^0\rangle_\mathbb{L}$ and $\langle x,v\rangle_\mathbb{L}=\langle x,v\rangle = \lambda \langle z,v\rangle = \lambda \langle z,v\rangle_\mathbb{L}$.
    Thus, the projection is
    \begin{equation}
        \begin{aligned}
            \Tilde{P}^v(x) &= \frac{1}{\sqrt{\langle x,x^0\rangle_\mathbb{L}^2-\langle x,v\rangle_\mathbb{L}^2}} \big(-\langle x,x^0\rangle_\mathbb{L} x^0 + \langle x,v\rangle_\mathbb{L} v\big) \\
            &= \frac{\lambda}{|\lambda|} \frac{1}{\sqrt{\langle z,x^0\rangle_\mathbb{L}^2-\langle z,v\rangle_\mathbb{L}^2}} \big(-\langle z,x^0\rangle_\mathbb{L} x^0 + \langle z,v\rangle_\mathbb{L} v\big) \\
            &= \frac{\lambda}{|\lambda|} z = \mathrm{sign}(\lambda) z.
        \end{aligned}
    \end{equation}
    But, $-z\notin \mathbb{L}^d$, hence necessarily, $\Tilde{P}^v(x) = z$.
    
    Finally, we can conclude that $\{x\in\mathbb{L}^d,\ \Tilde{P}^v(x)=z\} = \mathrm{span}(v_z)^\bot \cap \mathbb{L}^d$.
\end{proof}

From the previous proposition, we see that the Radon transform $\Bar{R}$ integrates over hyperplanes intersected with $\mathbb{L}^d$, which are totally geodesic submanifolds. This corresponds actually to the hyperbolic Radon transform first introduced by \citet{helgason1959differential} and studied more recently for example in \citep{berenstein1999radon, rubin2002radon, berenstein2004totally}. However, to the best of our knowledge, its injectivity over the set of measures has not been studied yet.

\paragraph{Radon transform for HHSW.} We can derive a Radon transform associated to HHSW in the same way. Moreover, the integration set can be intuitively derived as the level set of the Busemann function, since we project on the only point on the geodesic which has the same Busemann coordinates. Since the level sets of the Busemann functions correspond to horospheres, the associate Radon transform is the horospherical Radon transform. It has been for example studied by \citet{bray1999inversion, bray2019radon} on the Lorentz model, and by \citet{casadio2021radon} on the Poincaré ball. Note that it is also known as the Gelfand-Graev transform \citep{gelfand1966generalized}.

\subsection{Statistical Properties} \label{appendix:sample_complexity}

\paragraph{Sample Complexity.}

By adapting the proof of \citep[Corollary 2]{nadjahi2020statistical}, we derive a sample complexity in \cref{prop:sample_complexity} for both $GHSW_p$ and $HHSW_p$. Interestingly, they are similar up to some constant. Moreover, similarly as the Euclidean SW distance, they are independent of the dimension.

\begin{proposition} \label{prop:sample_complexity}
    Let $p\ge 1$, $q>p$ and $\mu,\nu\in\mathcal{P}_p(\mathbb{L}^d)$. Denote $\hat{\mu}_n$ and $\hat{\nu}_n$ their counterpart empirical measures and $M_q(\mu) = \int_{\mathbb{L}^d} d(x,x^0)^q\ \mathrm{d}\mu(x)$ the moments of order $q$. Then, there exists $C_{p,q}$ a constant depending only on $p$ and $q$ such that
    \begin{equation}
        \mathbb{E}\left[|GHSW_p(\hat{\mu}_n,\hat{\nu}_n) - GHSW_p(\mu,\nu)|\right] \le 2^{q/p} C_{p,q}^{1/p} (M_q(\mu)^{1/q}+M_q(\nu)^{1/q}) 
        \begin{cases}
            n^{-1/(2p)} \ \text{ if } q>2p \\
            n^{-1/(2p)}\log(n)^{1/p} \ \text{ if } q=2p \\
            n^{-(q-p)/(pq)} \ \text{ if } q\in (p,2p).
        \end{cases}
    \end{equation}
    Similarly, with $C_{p,q}$ a possible another constant,
    \begin{equation}
        \mathbb{E}\left[|HHSW_p(\hat{\mu}_n,\hat{\nu}_n) - HHSW_p(\mu,\nu)|\right] \le 2C_{p,q}^{1/p} (M_q(\mu)^{1/q}+M_q(\nu)^{1/q}) 
        \begin{cases}
            n^{-1/(2p)} \ \text{ if } q>2p \\
            n^{-1/(2p)}\log(n)^{1/p} \ \text{ if } q=2p \\
            n^{-(q-p)/(pq)} \ \text{ if } q\in (p,2p).
        \end{cases}
    \end{equation}
\end{proposition}

\begin{proof}
    For this proof, we first need to recall the following lemma adapted from \citep[Theorem 2]{fournier2015rate} and reported \emph{e.g.} in \citep{rakotomamonjy2021statistical}.
    \begin{lemma}[Lemma 1 in \citep{rakotomamonjy2021statistical}] \label{lemma:fournier}
        Let $p\ge 1$ and $\eta\in\mathcal{P}_p(\mathbb{R})$. Denote $\Tilde{M}_q(\eta)=\int |x|^q\ \mathrm{d}\eta(x)$ the moments of order $q$ and assume that $M_q(\eta)<\infty$ for some $q>p$. Then, there exists a constant $C_{p,q}$ depending only on $p,q$ such that for all $n\ge 1$,
        \begin{equation}
            \mathbb{E}[W_p^p(\hat{\eta}_n,\eta)] \le C_{p,q} \Tilde{M}_q(\eta)^{p/q}\left(n^{-1/2}\mathbb{1}_{\{q>2p\}} + n^{-1/2}\log(n) \mathbb{1}_{\{q=2p\}} + n^{-(q-p)/q} \mathbb{1}_{\{q\in(p,2p)\}}\right) \enspace .
        \end{equation}
    \end{lemma}
    
    Now, we will first deal with $GHSW_p$. First, let us observe that by the triangular and reverse triangular inequalities, as well as Jensen for $x\mapsto x^{1/p}$ (which is concave since $p\ge 1$),
    \begin{equation}
        \begin{aligned}
            \mathbb{E}\left[|\mathrm{GHSW}_p(\hat{\mu}_n,\hat{\nu}_n) - \mathrm{GHSW}_p(\mu,\nu)|\right] &= \mathbb{E}[|\mathrm{GHSW}_p(\hat{\mu}_n,\hat{\nu}_n)-\mathrm{GHSW}_p(\hat{\mu}_n,\nu) + \mathrm{GHSW}_p(\hat{\mu}_n,\nu) - \mathrm{GHSW}_p(\mu,\nu)|] \\
            &\le \mathbb{E}[|\mathrm{GHSW}_p(\hat{\mu}_n,\hat{\nu}_n) - \mathrm{GHSW}_p(\hat{\mu}_n,\nu)|] + \mathbb{E}[|\mathrm{GHSW}_p(\hat{\mu}_n,\nu)-\mathrm{GHSW}_p(\mu,\nu)|] \\
            &\le \mathbb{E}[\mathrm{GHSW}_p(\hat{\nu}_n,\nu)] + \mathbb{E}[\mathrm{GHSW}_p(\hat{\mu}_n,\mu)] \\
            &\le \mathbb{E}[\mathrm{GHSW}_p^p(\hat{\nu}_n,\nu)]^{1/p} + \mathbb{E}[\mathrm{GHSW}_p^p(\hat{\mu}_n,\mu)]^{1/p}.
        \end{aligned}
    \end{equation}
    Moreover, by Fubini-Tonelli,
    \begin{equation}
        \begin{aligned}
            \mathbb{E}[\mathrm{GHSW}_p^p(\hat{\mu}_n,\mu)] &= \mathbb{E}\left[\int_{T_{x^0}\mathbb{L}^d \cap S^d} W_p^p(P^v_\#\hat{\mu}_n,P^v_\#\mu)\ \mathrm{d}\lambda(v)\right] \\
            &= \int_{T_{x^0}\mathbb{L}^d\cap S^d} \mathbb{E}[W_p^p(P^v_\#\hat{\mu}_n,P^v_\#\mu)]\ \mathrm{d}\lambda(v).
        \end{aligned}
    \end{equation}
    By applying Lemma \ref{lemma:fournier}, we get for $q>p$ that there exists a constant $C_{p,q}$ such that,
    \begin{equation}
        \mathbb{E}[W_p^p(P^v_\#\hat{\mu}_n, P^v_\#\mu)] \le C_{p,q} \Tilde{M}_q(P^v_\#\mu)^{p/q}\left(n^{-1/2}\mathbb{1}_{\{q>2\}} + n^{-1/2}\log(n) \mathbb{1}_{\{q=2p\}} + n^{-(q-p)/q} \mathbb{1}_{\{q\in(p,2p)\}}\right).
    \end{equation}
    Furthermore, using \citep[Defintion 6.4]{villani2009optimal}, \emph{i.e.} that
    \begin{equation}
        \forall x,y,x_0\in\mathbb{L}^d,\ d_{\mathbb{L}}(x,y)^p \le 2^{p-1}\left(d_{\mathbb{L}}(x,x_0) + d_\mathbb{L}(x_0,y)\right),
    \end{equation}
    we obtain 
    \begin{equation}
        d(\Tilde{P}^v(x),x^0)^q \le 2^{q-1}\left(d(\Tilde{P}^v(x),x)^q +d(x,x^0)^q\right),
    \end{equation}
    and by definition of $\Tilde{P}^v$, $d_{\mathbb{L}}(\Tilde{P}^v(x),x) \le d_{\mathbb{L}}(x^0,x)$. Hence, remembering that $t^v(x) = \mathrm{sign}(\langle x, v\rangle) d_{\mathbb{L}}(x,x^0)$ and $P^v(x) = t^v(\Tilde{P}^v(x))$, we have
    \begin{equation}
        \begin{aligned}
            \Tilde{M}_q(P^v_\#\mu) &= \int_{\mathbb{R}} |x|^q\ \mathrm{d}(P^v_\#\mu)(x) \\
            &= \int_{\mathbb{L}^d} |P^v(x)|^q\ \mathrm{d}\mu(x) \\
            &= \int_{\mathbb{L}^d} |t^v(\Tilde{P}^v(x))|^q\ \mathrm{d}\mu(x) \\
            &= \int_{\mathbb{L}^d} d_{\mathbb{L}}(\Tilde{P}^v(x),x^0)^q\ \mathrm{d}\mu(x) \\
            &\le 2^{q-1} \left(\int_{\mathbb{L}^d} d_{\mathbb{L}}(\Tilde{P}^v(x), x)^q\ \mathrm{d}\mu(x) + \int_{\mathbb{L}^d} d_{\mathbb{L}}(x,x^0)^q\ \mathrm{d}\mu(x) \right) \\
            &\le 2^q \int_{\mathbb{L}^q} d_{\mathbb{L}}(x,x^0)^q\ \mathrm{d}\mu(x) = 2^q M_q(\mu).
        \end{aligned}
    \end{equation}
    Therefore, we have that
    \begin{equation}
        \begin{aligned}
            \mathbb{E}[\mathrm{GHSW}_p^p(\hat{\mu}_n,\mu)] \le 2^q C_{p,q} M_q(\mu)^{p/q} \left(n^{-1/2}\mathbb{1}_{\{q>2p\}} + n^{-1/2}\log(n) \mathbb{1}_{\{q=2p\}} + n^{-(q-p)/q} \mathbb{1}_{\{q\in(p,2p)\}}\right),
        \end{aligned}
    \end{equation}
    and similarly
    \begin{equation}
        \begin{aligned}
            \mathbb{E}[\mathrm{GHSW}_p^p(\hat{\nu}_n,\nu)] \le C_{p,q} M_q(\nu)^{p/q} \left(n^{-1/2}\mathbb{1}_{\{q>2p\}} + n^{-1/2}\log(n) \mathbb{1}_{\{q=2p\}} + n^{-(q-p)/q} \mathbb{1}_{\{q\in(p,2p)\}}\right).
        \end{aligned}
    \end{equation}
    Hence, we conclude that the sample complexity is 
    \begin{equation}
        \mathbb{E}\left[|\mathrm{GHSW}_p(\hat{\mu}_n,\hat{\nu}_n) - \mathrm{GHSW}_p(\mu,\nu)|\right] \le 2^{q/p} C_{p,q}^{1/p} \big(M_q(\mu)^{1/q} + M_q(\nu)^{1/q}\big) 
        \begin{cases}
            n^{-1/(2p)} \ \text{ if } q>2p \\
            n^{-1/(2p)}\log(n)^{1/p} \ \text{ if } q=2p \\
            n^{-(q-p)/(pq)} \ \text{ if } q\in (p,2p).
        \end{cases}
    \end{equation}
    
    Now, we can also do the same proof for $HHSW_p$. By using pseudo-distance properties, we also get 
    \begin{equation}
        \mathbb{E}\left[|\mathrm{HHSW}_p(\hat{\mu}_n,\hat{\nu}_n) - \mathrm{HHSW}_p(\mu,\nu)|\right] \le  \mathbb{E}[\mathrm{HHSW}_p^p(\hat{\nu}_n,\nu)]^{1/p} + \mathbb{E}[\mathrm{HHSW}_p^p(\hat{\mu}_n,\mu)]^{1/p},
    \end{equation}
    and with Fubini-Tonelli,
    \begin{equation}
        \mathbb{E}[HHSW_p^p(\hat{\mu}_n,\mu)] = \int_{S^{d-1}} \mathbb{E}[W_p^p(t^v_\#\Tilde{P}^v_\#\hat{\mu}_n, t^v_\#\Tilde{P}^v_\#\mu)]\ \mathrm{d}\lambda(v).
    \end{equation}
    Then, by Lemma \ref{lemma:fournier}, we have that for $q>p$, there exists a constant $C_{p,q}$ such that,
    \begin{equation}
        \begin{aligned}
            \mathbb{E}[W_p^p(B^v_\#\hat{\mu}_n,B^v_\#\mu)] &\le C_{p,q} \Tilde{M}_q(B^v_\#\mu)^{p/q}\left(n^{-1/2}\mathbb{1}_{\{q>2\}} + n^{-1/2}\log(n) \mathbb{1}_{\{q=2p\}} + n^{-(q-p)/q} \mathbb{1}_{\{q\in(p,2p)\}}\right).
        \end{aligned}
    \end{equation}
    But, as $B^v$ is 1-Lipschitz, and $B^v(x^0)=0$, we have that $|B^v(x)-B^v(x^0)|=|B^v(x)|\le d_{\mathbb{L}}(x,x^0)$ for all $x\in\mathbb{L}^d$. Hence,
    \begin{equation}
        \Tilde{M}_q(B^v_\#\mu) = \int_{\mathbb{L}^d} |B^v(x)|^q\ \mathrm{d}\mu(x) \le \int_{\mathbb{L}^d} d_{\mathbb{L}}(x,x^0)^q\ \mathrm{d}\mu(x) = M_q(\mu),
    \end{equation}
    and therefore
    \begin{equation}
        \mathbb{E}\left[|HHSW_p(\hat{\mu}_n,\hat{\nu}_n) - HHSW_p(\mu,\nu)|\right] \le 2C_{p,q}^{1/p} (M_q(\mu)^{1/q}+M_q(\nu)^{1/q}) 
        \begin{cases}
            n^{-1/(2p)} \ \text{ if } q>2p \\
            n^{-1/(2p)}\log(n)^{1/p} \ \text{ if } q=2p \\
            n^{-(q-p)/(pq)} \ \text{ if } q\in (p,2p).
        \end{cases}
    \end{equation}
\end{proof}

\paragraph{Projection Complexity.}

The integral \emph{w.r.t} the uniform measure on $S^{d-1}$ is unfortunately intractable, and therefore is required to be approximated by a Monte-Carlo scheme. In Proposition \ref{prop:projection_complexity}, we report the Monte-Carlo error of this approximation. We call this error the projection complexity. We recover here the same rate as \citep{nadjahi2020statistical} in the Euclidean case. Since the proposition and the proof are the same for $GHSW_p$ and $HHSW_p$, we do it for both in the same time by denoting $HSW_p$ in place of $GHSW_p$ or $HHSW_p$, and denoting by $P^v$ the corresponding projection with $v\in T_{x^0}\mathbb{L}^d\cap S^d$.

\begin{proposition}
    \label{prop:projection_complexity}
    Let $p\ge 1$, $\mu,\nu\in \mathcal{P}_{p}(\mathbb{L}^d)$. We denote $HSW_p$ for both $HHSW_p$ and $GHSW_p$. Then, the error made by the Monte Carlo estimate of $\mathrm{HSW}_p$ with L projections can be bounded as follows
    \begin{equation}
        \begin{aligned}
            \mathbb{E}_v\left[|\widehat{\mathrm{HSW}}_{p,L}^p(\mu,\nu)-\mathrm{HSW}_p^p(\mu,\nu)|\right]^2 
            &\le \frac{1}{L} \int_{T_{x^0}\mathbb{L}^d\cap S^d} \left(W_p^p(P^v_\#\mu, P^v_\#\nu)-\mathrm{HSW}_p^p(\mu,\nu)\right)^2 \ \mathrm{d}\lambda(v) \\
            &= \frac{1}{L}\mathrm{Var}_{v \sim \lambda}\left[W_p^p(P^v_\#\mu, P^v_\#\nu)\right],
        \end{aligned}
    \end{equation}
    where $\widehat{\mathrm{HSW}}_{p,L}^p(\mu,\nu) = \frac{1}{L}\sum_{i=1}^L W_p^p(P^{v_i}_\#\mu, P^{v_i}_\#\nu)$ with $(v_i)_{i=1}^L$ independent samples from $\lambda$.
\end{proposition}

\begin{proof}
    Let $(v_i)_{i=1}^L$ be iid samples of $\lambda$. Then, by first using Jensen inequality and then remembering that $\mathbb{E}_v[W_p^p(P^v_\#\mu, P^v_\#\nu)]=\mathrm{HSW}_p^p(\mu,\nu)$, we have
    \begin{equation}
        \begin{aligned}
            \mathbb{E}_v\left[|\widehat{\mathrm{HSW}}_{p,L}^p(\mu,\nu)-\mathrm{HSW}_p^p(\mu,\nu)|\right]^2 &\le \mathbb{E}_v\left[\left|\widehat{\mathrm{HSW}}_{p,L}^p(\mu,\nu)-\mathrm{HSW}_p^p(\mu,\nu)\right|^2\right]\\
            &= \mathbb{E}_v\left[\left|\frac{1}{L} \sum_{i=1}^L \big(W_p^p(P^{v_i}_\#\mu, P^{v_i}_\#\nu) - \mathrm{HSW}_p^p(\mu,\nu)\big)\right|^2\right] \\
            &= \frac{1}{L^2} \mathrm{Var}_v\left[\sum_{i=1}^L W_p^p(P^{v_i}_\#\mu, P^{v_i}_\#\nu)\right] \\
            &= \frac{1}{L} \mathrm{Var}_v\left[W_p^p(P^v_\#\mu, P^v_\#\nu)\right] \\
            &= \frac{1}{L} \int_{T_{x^0}\mathbb{L}^d\cap S^d} \left(W_p^p(P^v_\#\mu, P^v_\#\nu)-\mathrm{HSW}_p^p(\mu,\nu)\right)^2\ \mathrm{d}\lambda(v).
        \end{aligned}
    \end{equation}
\end{proof}

\section{Hyperbolic Spaces} \label{appendix:hyperbolic_space}

In this Section, we first recall different generalization of the Gaussian distribution on Hyperbolic spaces, with a particular focus on Wrapped normal distributions. Then, we recall how to perform Riemannian gradient descent in the Lorentz model and in the Poincaré ball.

\subsection{Distributions on Hyperbolic Spaces}

Let $M$ be a manifold and denote $G$ the corresponding Riemannian metric. For $x\in M$, $G(x)$ induces an infinitesimal change of volume on the tangent space $T_xM$, and thus a measure on the manifold,
\begin{equation*}
    \mathrm{d}\mathrm{Vol}(x) = \sqrt{|G(x)|}\ \mathrm{d}x.
\end{equation*}
We refer to \citep{pennec2006intrinsic} for more details on distributions on manifolds. Now, we recap different generalizations of Gaussian distribution on Riemannian manifolds.

\paragraph{Riemannian normal.} The first way of naturally generalizing Gaussian distributions to Riemannian manifolds is to use the geodesic distance in the density, which becomes
\begin{equation*}
    f(x) \propto \exp\left(-\frac{1}{2\sigma^2}d_M(x,\mu)^2\right).
\end{equation*}
It is actually the distribution maximizing the entropy \citep{pennec2006intrinsic, said2014new}. However, it is not straightforward to sample from such distribution. For example, \citet{ovinnikov2019poincar} use a rejection sampling algorithm.

\paragraph{Wrapped normal distribution.} A more convenient distribution, on which we can use the parameterization trick, is the Wrapped normal distribution \citep{nagano2019wrapped}. This distribution can be sampled from by first drawing $v\sim \mathcal{N}(0,\Sigma)$ and then transforming it into $v\in T_{x^0}\mathbb{L}^d$ by concatenating a 0 in the first coordinate. Then, we perform parallel transport to transport $v$ from the tangent space of $x^0$ to the tangent space of $\mu\in\mathbb{L}^d$. Finally, we can project the samples on the manifold using the exponential map. We recall the formula of parallel transport form $x$ to $y$:
\begin{equation}
    \forall v\in T_{x}\mathbb{L}^d,\ \mathrm{PT}_{x\to y}(v) = v + \frac{\langle y,v\rangle_\mathbb{L}}{1-\langle x,y\rangle_\mathbb{L}}(x+y).
\end{equation}

Since it only involves differentiable operations, we can perform the parameterization trick and \emph{e.g.} optimize directly over the mean and the variance. Moreover, by the change of variable formula, we can also derive the density \citep{nagano2019wrapped,bose2020latent}. Let $\Tilde{z}\sim\mathcal{N}(0,\Sigma)$, $z=(0,\Tilde{z})\in T_{x^0}\mathbb{L}^d$, $u=\mathrm{PT}_{x^0\to\mu}(z)$, then the density of $x=\exp_{\mu}(u)$ is:
\begin{equation}
    \log p(x) = \log p(\Tilde{z}) - (d-1)\log\left(\frac{\sinh(\|u\|_\mathbb{L})}{\|u\|_\mathbb{L}}\right).
\end{equation}
In the paper, we write $x\sim \mathcal{G}(\mu,\Sigma)$.

\subsection{Optimization on Hyperbolic Spaces} \label{appendix:optim}

For gradient descent on hyperbolic space, we refer to \citep[Section 7.6]{boumal2022intromanifolds} and \citep{wilson2018gradient}. 

In general, for a functional $f:M\to\mathbb{R}$, Riemannian gradient descent is performed, analogously to the Euclidean space, by following the geodesics. Hence, the gradient descent reads as \citep{absil2009optimization, bonnabel2013stochastic}
\begin{equation}
    \forall k\ge 0,\ x_{k+1} = \exp_{x_k}\big(-\gamma \mathrm{grad} f(x_k)\big).
\end{equation}
Note that the exponential map can be replaced more generally by a retraction. We describe in the following paragraphs the different formulae in the Lorentz model and in the Poincaré ball.

\paragraph{Lorentz model.}

Let $f:\mathbb{L}^d \to \mathbb{R}$, then its Riemannian gradient is \citep[Proposition 7.7]{boumal2022intromanifolds}
\begin{equation}
    \mathrm{grad}f(x) = \mathrm{Proj}_x(J\nabla f(x)),
\end{equation}
where $J=\mathrm{diag}(-1,1,\dots,1)$ and $\mathrm{Proj}_x(z) = z + \langle x,z\rangle_\mathbb{L} x$. Furthermore, the exponential map is
\begin{equation}
    \forall v\in T_x\mathbb{L}^d,\ \exp_x(v) = \cosh(\|v\|_\mathbb{L})x + \sinh(\|v\|_\mathbb{L}) \frac{v}{\|v\|_\mathbb{L}}.
\end{equation}

\paragraph{Poincaré ball.}

On $\mathbb{B}^d$, the Riemannian gradient of $f:\mathbb{B}^d\to\mathbb{R}$ can be obtained as \citep[Section 3]{nickel2017poincare}
\begin{equation}
    \mathrm{grad}f(x) = \frac{(1-\|\theta\|_2^2)^2}{4}\nabla f(x).
\end{equation}

\citet{nickel2017poincare} propose to use as retraction $R_x(v) = x+v$ instead of the exponential map, and add a projection, to constrain the value to remain within the Poincaré ball, of the form
\begin{equation}
    \mathrm{proj}(x) = \begin{cases}
        \frac{x}{\|x\|_2} - \epsilon \quad \text{if }\|x\|\ge 1 \\
        x \quad \text{otherwise},
    \end{cases}
\end{equation}
where $\epsilon=10^{-5}$ is a small constant ensuring numerical stability. Hence, the algorithm becomes
\begin{equation}
    x_{k+1} = \mathrm{proj}\left(x_k-\gamma_k \frac{(1-\|x_k\|_2^2)^2}{4} \nabla f(x_k)\right).
\end{equation}

A second solution is to compute directly the exponential map derived in  \citep[Corollary 1.1]{ganea2018cone}:
\begin{equation} \label{eq:exp_poincare}
    \exp_x(v) = \frac{\lambda_x \big(\cosh(\lambda_x\|v\|_2) + \langle x, \frac{v}{\|v\|_2}\rangle \sinh(\lambda_x \|v\|_2)\big) x + \frac{1}{\|v\|_2} \sinh(\lambda_x \|v\|_2)v}{1+(\lambda_x -1)\cosh(\lambda_x\|v\|_2) + \lambda_x \langle x, \frac{v}{\|v\|_2}\rangle \sinh(\lambda_x \|v\|_2)},
\end{equation}
where $\lambda_x = \frac{2}{1-\|x\|_2^2}$.

\section{Additional Details and Experiments} \label{appendix:xps}

\subsection{Comparisons}

\begin{figure*}[t]
    \centering
    \hspace*{\fill}
    \subfloat[$\tau=0.05$]{\label{fig:tree_0.05}\includegraphics[width={0.23\linewidth}]{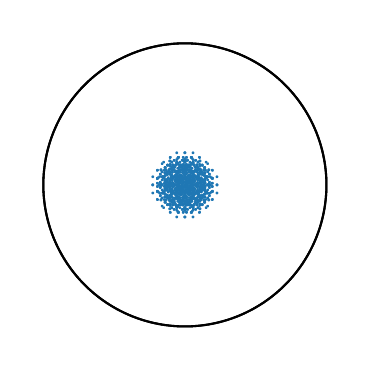}} \hfill
    \subfloat[$\tau=0.25$]{\label{fig:tree_0.25}\includegraphics[width={0.23\linewidth}]{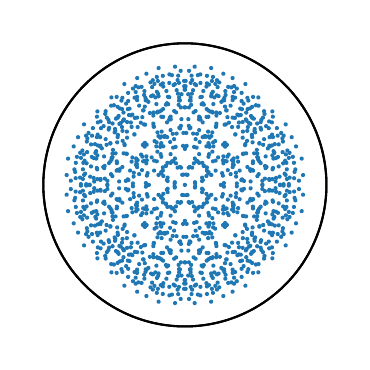}} \hfill
    \subfloat[$\tau=0.5$]{\label{fig:tree_0.50}\includegraphics[width={0.23\linewidth}]{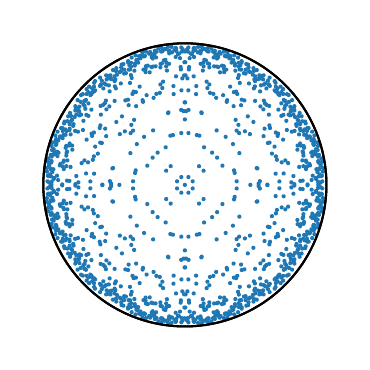}} \hfill
    \subfloat[$\tau=0.8$]{\label{fig:tree_0.80}\includegraphics[width=0.23\linewidth]{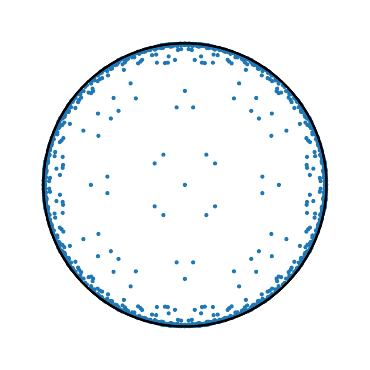}}
    \hspace*{\fill}
    \caption{Embeddings of trees using Sarkar's algorithm with different $\tau$.}
    \label{fig:trees_sarkar}
\end{figure*}

\begin{figure*}[t]
    \centering
    \hspace*{\fill}
    \subfloat[SW on Poincaré (SWp), GHSW]{\label{fig:w_wnd_tree}\includegraphics[width=0.3\linewidth]{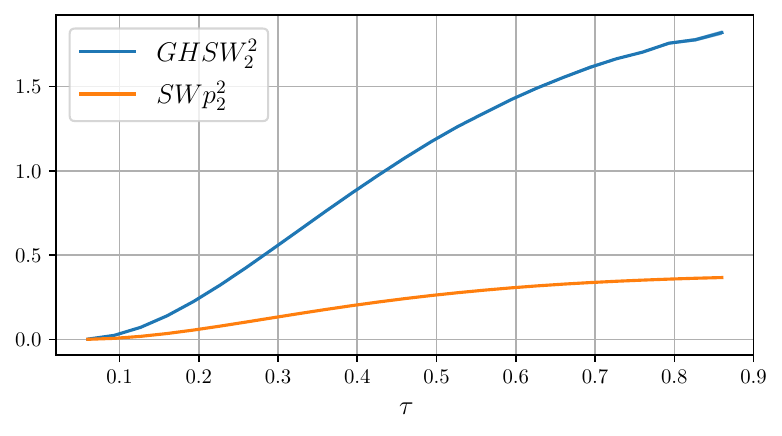}} \hfill
    \subfloat[HHSW and Wasserstein]{\label{fig:hsw_wnd_tree}\includegraphics[width=0.3\linewidth]{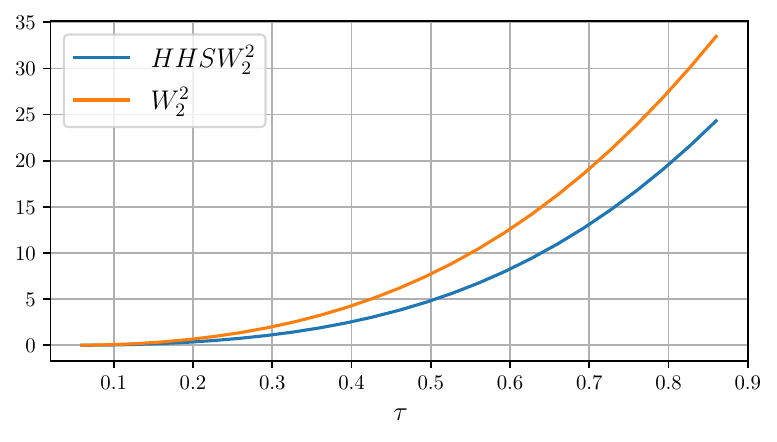}} \hfill
    \subfloat[SW on Lorentz (SWl)]{\label{fig:sw_wnd_lorentz_tree}\includegraphics[width=0.3\linewidth]{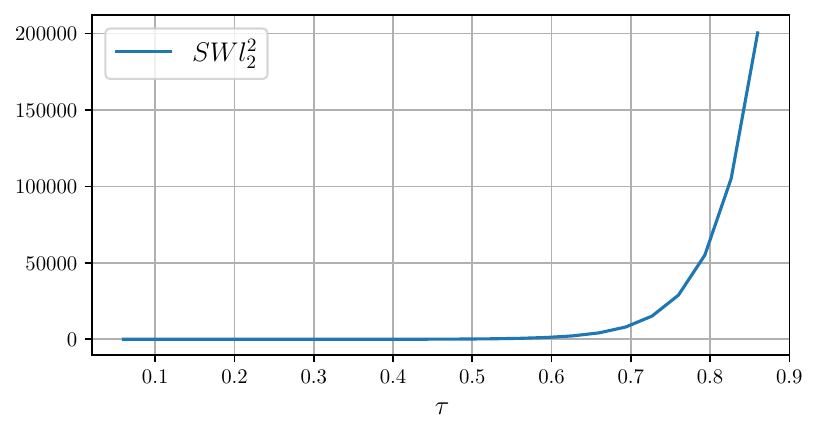}} \hfill
    \hspace*{\fill} \\
    \caption{Comparison of the Wasserstein distance (with the geodesic distance as cost), GHSW, HHSW and SW between embedded trees. We gather the discrepancies together by scale of the values.}
    \label{fig:comparison_trees}
\end{figure*}

In Section \ref{section:comparison_sw}, we compare the evolution of GHSW, HHSW, SWl, SWp and the Wasserstein distance with geodesic cost between wrapped normal distributions. Here, we add a more ``hyperbolical'' setting in the sense that we compare trees embedded in hyperbolic space. Indeed, it is well known that hyperbolic spaces can be seen as a continuous analog of trees, and are therefore a natural embedding space for trees. 

More precisely, we generate balanced trees using NetworkX \citep{networkx} and embed them with Sarkar's algorithm \citep{sarkar2011low,sala2018representation}. This algorithm takes as input a scaling factor $\tau$ which determines how close to the border will the leaves be. We illustrate such embeddings with different $\tau$ on Figure \ref{fig:trees_sarkar}. We compare in Figure \ref{fig:comparison_trees} the evolution of GHSW, HHSW, SWl and SWp between a tree embedded very close to the origin with $\tau=0.05$ and $\tau$ growing towards 1. We observe here the same evolution than in Section \ref{section:comparison_sw}.

\paragraph{Sample complexity.}


\begin{figure}
    \centering
    \includegraphics[width=0.35\linewidth]{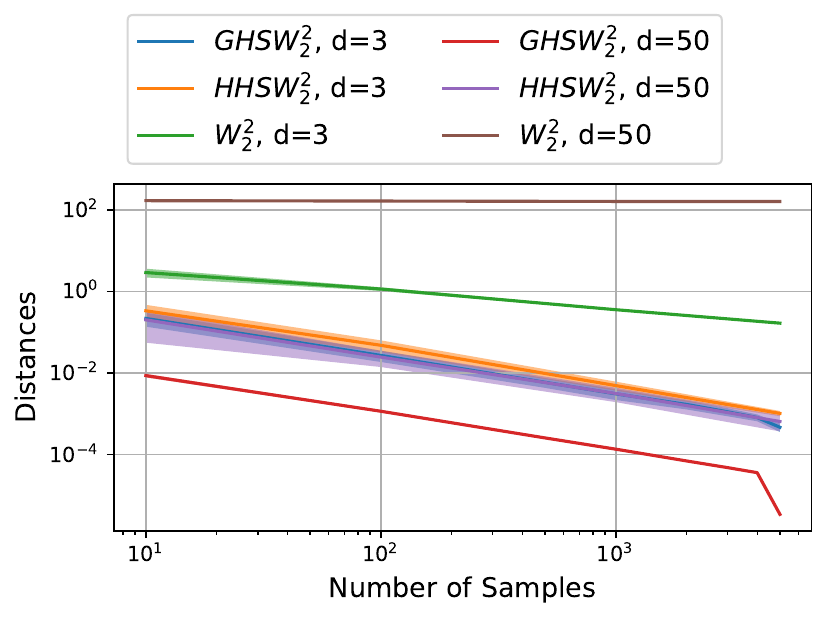}
    \caption{Sample complexity of GHSW, HHSW and Wasserstein with geodesic distance. GHSW and HHSW have the same convergence rate in dimension 3 and 50.}
    \label{fig:sample_complexity}
\end{figure}

We showed in \cref{prop:sample_complexity} that the sample complexity of $HHSW_p$ and $GHSW_p$ does not depend on the dimension. We verify here on Figure \ref{fig:sample_complexity} empirically this property for GHSW and HHSW between two set of samples drawn from $\mathcal{G}(x^0,I_2)$, and computed with 1000 projections. In dimension 3 and 50, HHSW and GHSW have the same convergence speed \emph{w.r.t.} the number of samples, which is not the case for the Wasserstein distance which suffers from the curse of dimensionality.

\subsection{Gradient flows.} \label{appendix:gradient_flows}

Denoting $\nu$ the target distribution from which we have access to samples $(y_i)_{i=1}^m$, we aim at learning $\nu$ by solving the following optimization problem:
\begin{equation}
    \mu = \argmin_{\mu}\ HSW\left(\mu, \frac{1}{m}\sum_{i=1}^m \delta_{x_i}\right).
\end{equation}
As we cannot directly learn $\mu$, we model it as $\hat{\mu}=\frac{1}{n}\sum_{i=1}^n \delta_{x_i}$, and then learn the sample locations $(x_i)_{i=1}^n$ using a Riemannian gradient descent which we described in Appendix \ref{appendix:optim}. In practice, we take $n=500$ and use batchs of $500$ target samples at each iteration. To compute the sliced discrepancies, we always use 1000 projections. On Figure \ref{fig:comparison_gradientflows}, we plot the log 2-Wasserstein with geodesic cost between the model measure $\hat{\mu}_k$ at each iteration $k$ and $\nu$. We average over 5 runs of each gradient descent. Now, we describe the specific setting for the different targets.

\paragraph{Wrapped normal distribution.} For the first experiment, we choose as target a wrapped normal distribution $\mathcal{G}(m,\Sigma)$. In the fist setting, we use $m=(1.5,1.25,0)\in\mathbb{L}^2$ and $\Sigma=0.1 I_2$. In the second, we use $m=(8,\sqrt{63},0)\in\mathbb{L}^2$ and $\Sigma = 0.1 I_2$. The learning rate is fixed as 5 for the different discrepancies, except for SWl on the second WND which lies far from origin, and for which we exhibit numerical instabilities with a learning rate too high. Hence, we reduced it to 0.1. We observed the same issue for HHSW on the Lorentz model. Fortunately, the Poincaré version, which is equal to the Lorentz version, did not suffer from these issues. It underlines the benefit of having both formulations. 

On Figure \ref{fig:evolution_hsw_particles}, we plotted the evolution of the particles for HHSW and GHSW with a target with mean $m=(8,\sqrt{63},0)$ and $\Sigma=\begin{pmatrix}1 & \frac12 \\ \frac12 & 1\end{pmatrix}$. For GHSW, we use a learning rate of 10, and for HSHW a learning rate of 100. We observe that the trajectories are differents. With geodesic projections, the particles go towards the target by passing through the origin, while with horospherical projections, the tend first to leave the origin.

\begin{figure}[t]
    \centering
    \hspace*{\fill}
    \subfloat[With geodesic projection.]{\label{fig:evolution_ghsw_particles}\includegraphics[width=\columnwidth]{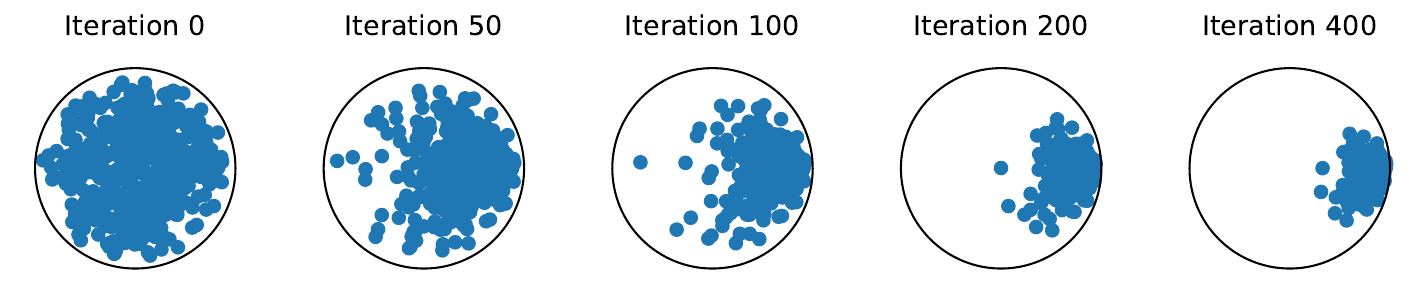}} 
    \hspace*{\fill} \\
    \hspace*{\fill}
    \subfloat[With horospherical projection.]{\label{fig:evolution_hhsw_partices}\includegraphics[width=\columnwidth]{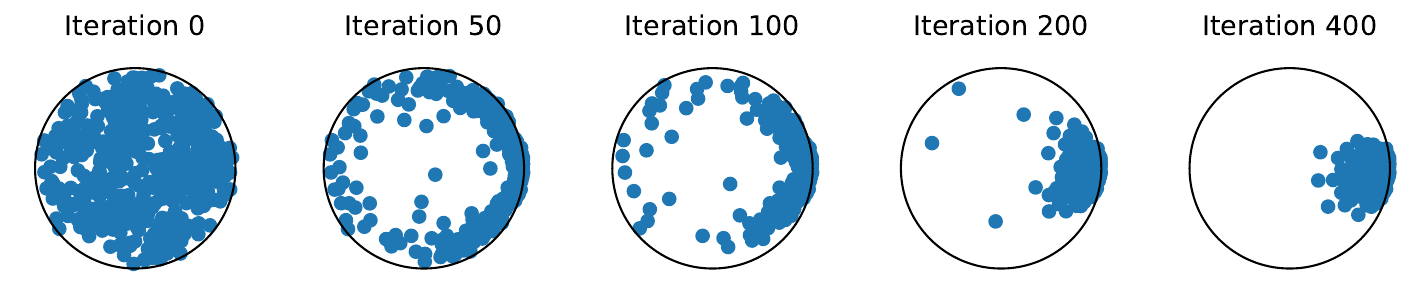}}
    \hspace*{\fill}
    \caption{Evolution of the particles along the gradient flow of HSW (with geodesic or horospherical projection).}
    \label{fig:evolution_hsw_particles}
\end{figure}

\paragraph{Mixture of wrapped normal distributions.} For the second experiment, the target is a mixture of 5 WNDs. The covariance are all taken equal as $0.01 I_2$. For the first setting, the outlying means are (on the Poincaré ball) $m_1=(0,-0.5)$, $m_2=(0,0.5)$, $m_3=(0.5,0)$, $m_4=(-0.5,0)$ and the center mean is $m_5 = (0,0.1)$. In the second setting, the outlying means are $m_1=(0,-0.9)$, $m_2=(0,0.9)$, $m_3=(0.9,0)$ and $m_4=(-0.9,0)$. We use the same $m_5$. The learning rate in this experiment is fixed at 1 for all discrepancies.

\subsection{Classification of Images with Busemann}

Denote $\{(x_i,y_i)_{i=1}^n\}$ the training set where $x_i\in\mathbb{R}^m$ and $y_i\in\{1,\dots,C\}$ is a label. The embedding is performed by using a neural network $f_\theta$ and the exponential map at the last layer, which projects the points on the Poincaré ball, \emph{i.e.} for $i\in\{1,\dots,n\}$, the embedding of $x_i$ is $z_i = \exp_0\big(f_\theta(z_i)\big)$, where $\exp_0$ is given by \eqref{eq:exp_poincare}, or more simply by
\begin{equation}
    \exp_0(x) = \tanh\left(\frac{\|x\|_2}{2}\right)\frac{x}{\|x\|_2}.
\end{equation}

The experimental setting of this experiment is the same as \citep{ghadimi2021hyperbolic}. That is, we use a Resnet-32 backbone and optimize it with Adam \citep{kingma2014adam}, a learning rate of 5e-4, weight decay of 5e-5, batch size of 128 and without pre-training. The network is trained for all experiments for 1110 epochs with learning rate decay of 10 after 1000 and 1100 epochs. Moreover, the $C$ prototypes are given by the algorithm of \citep{mettes2019hyperspherical} and are uniform on the sphere $S^{d-1}$.

For the additional hyperparameters in the loss \eqref{eq:loss_hsw}, we use by default $\lambda = 1$, and a mixture of $C$ wrapped normal distributions with means $\alpha p_c$, where $p_c\in S^{d-1}$ is a prototype, $c\in\{1,\dots,C\}$ and $\alpha=0.75$, and covariance matrix $\sigma I_d$ with $\sigma=0.1$. The number of projection is by default set at L=1000. 


\subsection{Hyperbolic Sliced-Wassertein Autoencoder} \label{appendix:hswae}

\begin{figure}[htpb]
    \centering
    \includegraphics[width=\linewidth]{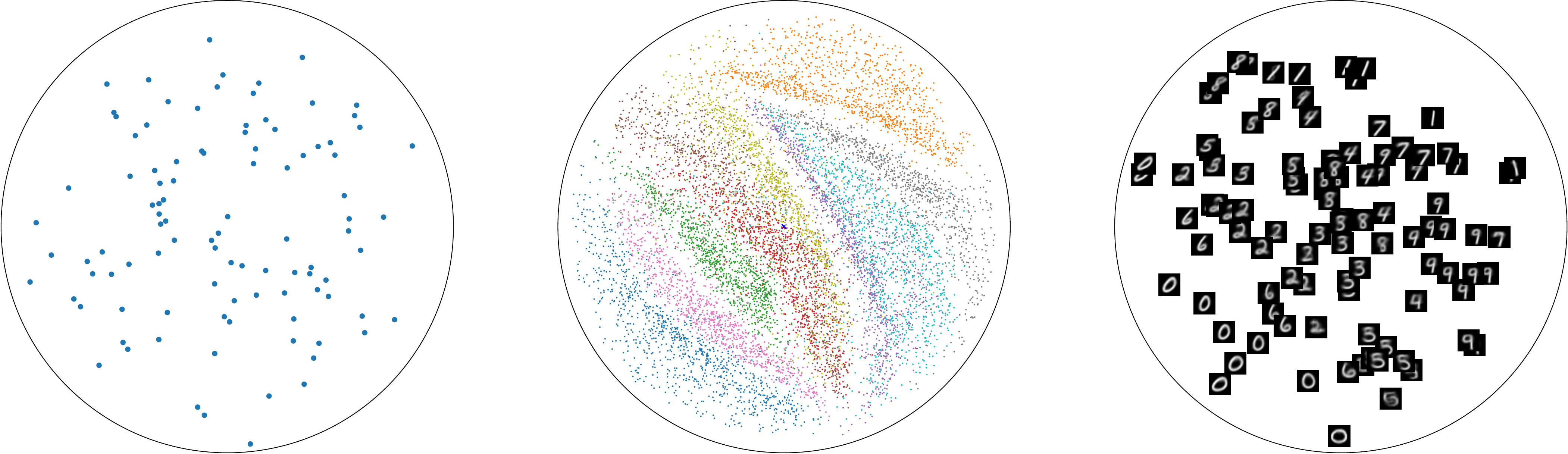} \\
    \vspace{10pt}
    \includegraphics[width=\linewidth]{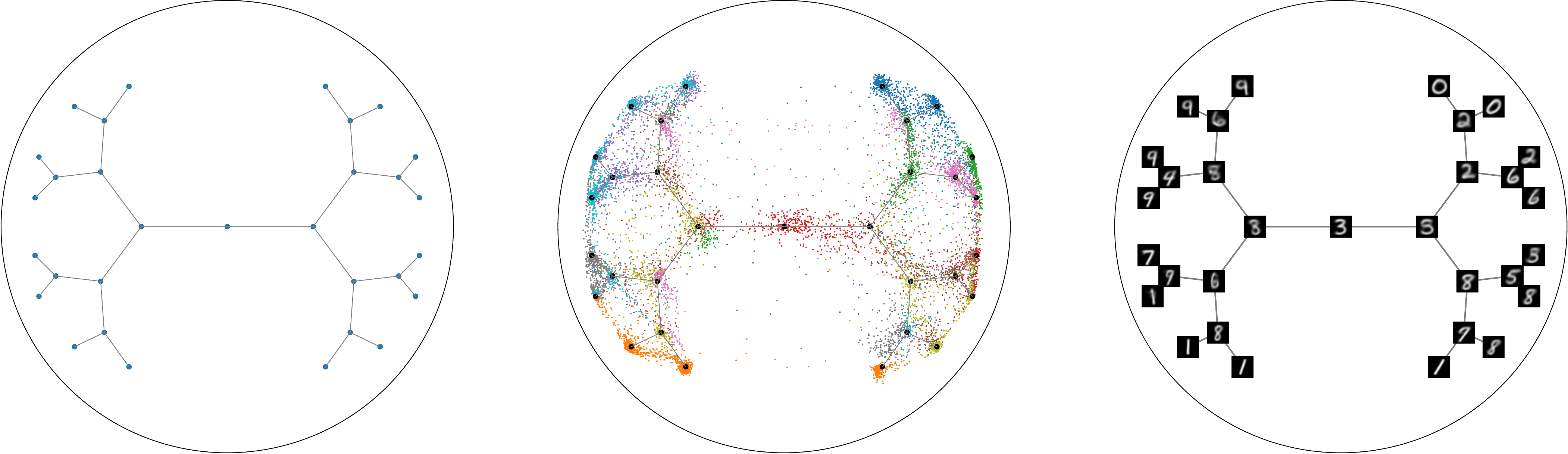} \\
    \vspace{10pt}
    \includegraphics[width=\linewidth]{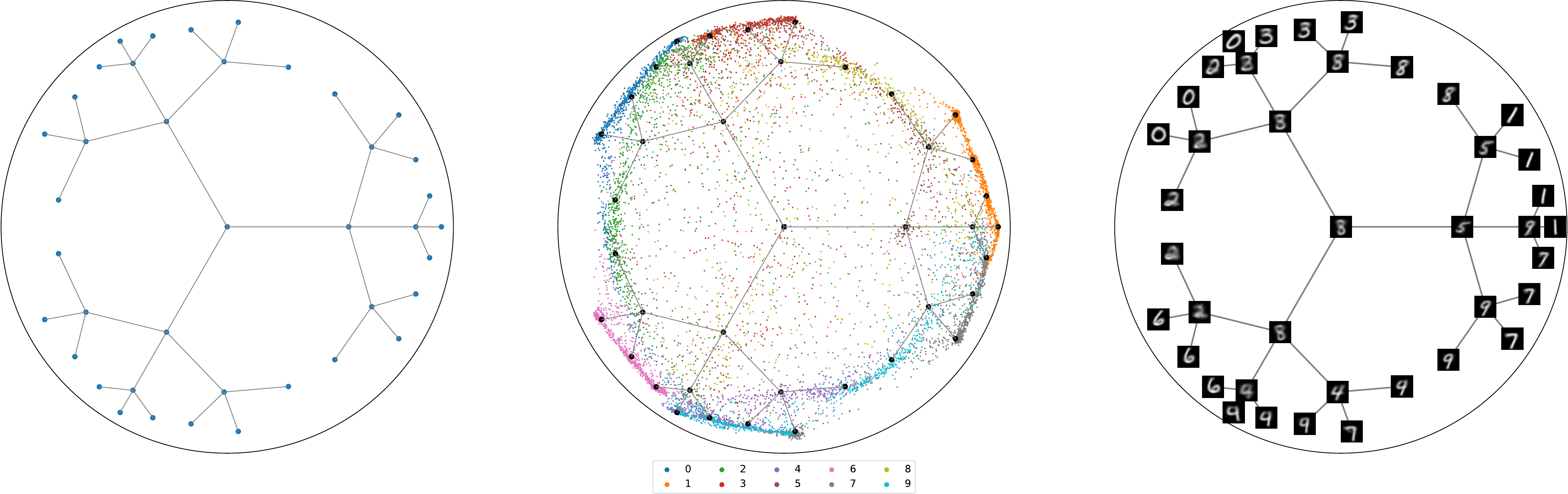}
    \caption{Embedding and reconstruction for HHSWAE. In the first column, we plot the prior. In the second column, we plot the embedding of MNIST and in the third column, we plot the reconstructed nodes of the tree or from samples of the wrapped normal distribution. In the first row, the prior is a Wrapped Normal Distribution. In the second row, the prior is a binary tree and in the third row a ternary tree.}
    \label{fig:hswae_tree}
\end{figure}

As hyperbolic spaces allow to embed hierarchical data, it has been proposed in several works to put a prior on such space for autoencoder tasks \citep{ovinnikov2019poincar,nagano2019wrapped,mathieu2019continuous}. Usually, an uninformative prior such as a Wrapped normal or a Riemannian normal distribution is used. For such distributions, the density is known and hence the Kullback-Leibler divergence can be approximated by a Monte-Carlo scheme. Moreover, we can also use the reparametrization trick. Then, a variational auto-encoder \citep{kingma2013auto} can be used. For more complicated distributions or deterministic prior with no density, we can use Wasserstein autoencoders \citep{tolstikhin2017wasserstein}. In this case, with a prior $p_Z$ for which we have access to samples, an encoder $f$ mapping the distribution data $\mu$ to the latent space, and a decoder $g$, we aim at minimizing the following loss:
\begin{equation}
    \mathcal{L}(f,g) = \int c(x,g(f(x)))\ \mathrm{d}\mu(x) + D(f_\#\mu, p_Z),
\end{equation}
with $c$ some cost function and $D$ some divergence. Several divergences $D$ were proposed such as the MMD or SW \citep{kolouri2018sliced}. We propose here to study the latent space when using a tree prior, for which we cannot use a variational autoencoder. To learn the distribution in the latent space, we use a hyperbolic sliced discrepancy.

On Figure \ref{fig:hswae_tree}, we compare several priors on the Mnist dataset \citep{lecun-mnisthandwrittendigit-2010} with $D=HHSW_2^2$, which we denote HHSWAE. First, we use a Wrapped Normal distribution, and then a binary and a ternary tree as a prior. The trees are generated with NetworkX and embedded using Sarkar's algorithm, with $\tau=0.6$ for the ternary tree and $\tau=0.4$ for the binary tree. Moreover, we use a height of 3 for the ternary tree and of 4 for the binary one. For the HHSWAE, we used 200 epochs with the same architectures as \citep{kolouri2018sliced} with an exp map before the output of the encoder, and a log map at the input of the decoder.

We observe that when using a tree prior, the points from the same class tend to be distributed around the same nodes. We believe that such an hierarchical prior can be beneficial in cases where one already has an assumption about the natural structure of the data. Next works will consider this question more thoroughly in different applicative settings.





\end{document}